\input{glyphtounicode}

\pdfglyphtounicode{angbracketleftBig}{27E8 FE02}
\pdfglyphtounicode{angbracketleftBigg}{27E8 FE04}
\pdfglyphtounicode{angbracketleftbig}{27E8 FE01}
\pdfglyphtounicode{angbracketleftbigg}{27E8 FE03}
\pdfglyphtounicode{angbracketrightBig}{27E9 FE02}
\pdfglyphtounicode{angbracketrightBigg}{27E9 FE04}
\pdfglyphtounicode{angbracketrightbig}{27E9 FE01}
\pdfglyphtounicode{angbracketrightbigg}{27E9 FE03}
\pdfglyphtounicode{arrowbt}{2193}
\pdfglyphtounicode{arrowdblbt}{21D3}
\pdfglyphtounicode{arrowdbltp}{21D1}
\pdfglyphtounicode{arrowhookleft}{21AA}
\pdfglyphtounicode{arrowhookright}{21A9}
\pdfglyphtounicode{arrowtp}{2191}
\pdfglyphtounicode{arrowvertex}{23D0}
\pdfglyphtounicode{arrowvertexdbl}{20E6}
\pdfglyphtounicode{backslashBig}{005C FE02}
\pdfglyphtounicode{backslashBigg}{005C FE04}
\pdfglyphtounicode{backslashbig}{005C FE01}
\pdfglyphtounicode{backslashbigg}{005C FE03}
\pdfglyphtounicode{braceex}{23AA}
\pdfglyphtounicode{bracehtipdownleft}{23DF}
\pdfglyphtounicode{bracehtipdownright}{23DF}
\pdfglyphtounicode{bracehtipupleft}{23DE}
\pdfglyphtounicode{bracehtipupright}{23DE}
\pdfglyphtounicode{braceleftBig}{007B FE02}
\pdfglyphtounicode{braceleftBigg}{007B FE04}
\pdfglyphtounicode{braceleftbig}{007B FE01}
\pdfglyphtounicode{braceleftbigg}{007B FE03}
\pdfglyphtounicode{braceleftbt}{23A9}
\pdfglyphtounicode{braceleftmid}{23A8}
\pdfglyphtounicode{bracelefttp}{23A7}
\pdfglyphtounicode{bracerightBig}{007D FE02}
\pdfglyphtounicode{bracerightBigg}{007D FE04}
\pdfglyphtounicode{bracerightbig}{007D FE01}
\pdfglyphtounicode{bracerightbigg}{007D FE03}
\pdfglyphtounicode{bracerightbt}{23AD}
\pdfglyphtounicode{bracerightmid}{23AC}
\pdfglyphtounicode{bracerighttp}{23AB}
\pdfglyphtounicode{bracketleftBig}{005B FE02}
\pdfglyphtounicode{bracketleftBigg}{005B FE04}
\pdfglyphtounicode{bracketleftbig}{005B FE01}
\pdfglyphtounicode{bracketleftbigg}{005B FE03}
\pdfglyphtounicode{bracketleftbt}{23A3}
\pdfglyphtounicode{bracketleftex}{23A2}
\pdfglyphtounicode{bracketlefttp}{23A1}
\pdfglyphtounicode{bracketrightBig}{005D FE02}
\pdfglyphtounicode{bracketrightBigg}{005D FE04}
\pdfglyphtounicode{bracketrightbig}{005D FE01}
\pdfglyphtounicode{bracketrightbigg}{005D FE03}
\pdfglyphtounicode{bracketrightbt}{23A6}
\pdfglyphtounicode{bracketrightex}{23A5}
\pdfglyphtounicode{bracketrighttp}{23A4}
\pdfglyphtounicode{ceilingleftBig}{2308 FE02}
\pdfglyphtounicode{ceilingleftBigg}{2308 FE04}
\pdfglyphtounicode{ceilingleftbig}{2308 FE01}
\pdfglyphtounicode{ceilingleftbigg}{2308 FE03}
\pdfglyphtounicode{ceilingrightBig}{2309 FE02}
\pdfglyphtounicode{ceilingrightBigg}{2309 FE04}
\pdfglyphtounicode{ceilingrightbig}{2309 FE01}
\pdfglyphtounicode{ceilingrightbigg}{2309 FE03}
\pdfglyphtounicode{circledotdisplay}{2A00 FE02}
\pdfglyphtounicode{circledottext}{2A00 FE01}
\pdfglyphtounicode{circlemultiplydisplay}{2A02 FE02}
\pdfglyphtounicode{circlemultiplytext}{2A02 FE01}
\pdfglyphtounicode{circleplusdisplay}{2A01 FE02}
\pdfglyphtounicode{circleplustext}{2A01 FE01}
\pdfglyphtounicode{contintegraldisplay}{222E FE02}
\pdfglyphtounicode{contintegraltext}{222E FE01}
\pdfglyphtounicode{coproductdisplay}{2210 FE02}
\pdfglyphtounicode{coproducttext}{2210 FE01}
\pdfglyphtounicode{floorleftBig}{230A FE02}
\pdfglyphtounicode{floorleftBigg}{230A FE04}
\pdfglyphtounicode{floorleftbig}{230A FE01}
\pdfglyphtounicode{floorleftbigg}{230A FE03}
\pdfglyphtounicode{floorrightBig}{230B FE02}
\pdfglyphtounicode{floorrightBigg}{230B FE04}
\pdfglyphtounicode{floorrightbig}{230B FE01}
\pdfglyphtounicode{floorrightbigg}{230B FE03}
\pdfglyphtounicode{hatwide}{02C6 FE01}
\pdfglyphtounicode{hatwider}{02C6 FE02}
\pdfglyphtounicode{hatwidest}{02C6 FE03}
\pdfglyphtounicode{integraldisplay}{222B FE02}
\pdfglyphtounicode{integraltext}{222B FE01}
\pdfglyphtounicode{intersectiondisplay}{22C2 FE02}
\pdfglyphtounicode{intersectiontext}{22C2 FE01}
\pdfglyphtounicode{logicalanddisplay}{22C0 FE02}
\pdfglyphtounicode{logicalandtext}{22C0 FE01}
\pdfglyphtounicode{logicalordisplay}{22C1 FE02}
\pdfglyphtounicode{logicalortext}{22C1 FE01}
\pdfglyphtounicode{mapsto}{21A6}
\pdfglyphtounicode{parenleftBig}{0028 FE02}
\pdfglyphtounicode{parenleftBigg}{0028 FE04}
\pdfglyphtounicode{parenleftbig}{0028 FE01}
\pdfglyphtounicode{parenleftbigg}{0028 FE03}
\pdfglyphtounicode{parenleftbt}{239D}
\pdfglyphtounicode{parenleftex}{239C}
\pdfglyphtounicode{parenlefttp}{239B}
\pdfglyphtounicode{parenrightBig}{0029 FE02}
\pdfglyphtounicode{parenrightBigg}{0029 FE04}
\pdfglyphtounicode{parenrightbig}{0029 FE01}
\pdfglyphtounicode{parenrightbigg}{0029 FE03}
\pdfglyphtounicode{parenrightbt}{23A0}
\pdfglyphtounicode{parenrightex}{239F}
\pdfglyphtounicode{parenrighttp}{239E}
\pdfglyphtounicode{productdisplay}{220F FE02}
\pdfglyphtounicode{producttext}{220F FE01}
\pdfglyphtounicode{radicalBig}{221A FE02}
\pdfglyphtounicode{radicalBigg}{221A FE04}
\pdfglyphtounicode{radicalbig}{221A FE01}
\pdfglyphtounicode{radicalbigg}{221A FE03}
\pdfglyphtounicode{radicalbt}{23B7}
\pdfglyphtounicode{radicaltp}{231C}
\pdfglyphtounicode{radicalvertex}{20D3}
\pdfglyphtounicode{slashBig}{002F FE02}
\pdfglyphtounicode{slashBigg}{002F FE04}
\pdfglyphtounicode{slashbig}{002F FE01}
\pdfglyphtounicode{slashbigg}{002F FE03}
\pdfglyphtounicode{summationdisplay}{2211 FE02}
\pdfglyphtounicode{summationtext}{2211 FE01}
\pdfglyphtounicode{tie}{2040}
\pdfglyphtounicode{tildewide}{02DC FE01}
\pdfglyphtounicode{tildewider}{02DC FE02}
\pdfglyphtounicode{tildewidest}{02DC FE03}
\pdfglyphtounicode{uniondisplay}{22C3 FE02}
\pdfglyphtounicode{unionmultidisplay}{2A04 FE02}
\pdfglyphtounicode{unionmultitext}{2A04 FE01}
\pdfglyphtounicode{unionsqdisplay}{2A06 FE02}
\pdfglyphtounicode{unionsqtext}{2A06 FE01}
\pdfglyphtounicode{uniontext}{22C3 FE01}
\pdfglyphtounicode{vextenddouble}{20E6}
\pdfglyphtounicode{vextendsingle}{20D3}

\newif\ifpubliccopy
\publiccopytrue
\documentclass{article}
\usepackage{iclr2027_conference,times}
\usepackage[T1]{fontenc}

\usepackage{amsmath,amssymb,amsthm,mathtools}
\usepackage{booktabs}
\usepackage{array}
\usepackage{multirow}
\usepackage{microtype}
\usepackage{graphicx}
\usepackage{xcolor}
\usepackage{float}
\usepackage{tikz}
\usetikzlibrary{arrows.meta,positioning,fit}
\usepackage[section]{placeins}
\usepackage{url}
\usepackage[hidelinks]{hyperref}

\newtheorem{theorem}{Theorem}
\newtheorem{proposition}[theorem]{Proposition}

\theoremstyle{definition}

\theoremstyle{remark}

\DeclareMathOperator{\supp}{supp}
\DeclareMathOperator{\TV}{TV}
\DeclareMathOperator{\ESS}{ESS}
\newcommand{\R}{\mathbb{R}}
\newcommand{\E}{\mathbb{E}}
\newcommand{\Pp}{\mathbb{P}}
\newcommand{\Qp}{\mathbb{Q}}

\newcommand{\direct}{Spectra-Direct}
\newcommand{\pathspace}{Spectra-PS}
\newcommand{\pgfullcov}{PG-FullCov}
\newcommand{\benchval}[2]{#1{\color[gray]{0.48}(#2)}}
\DeclareMathSizes{7.5}{8}{6}{5}
\newcounter{algorithm}
\newenvironment{paperalgorithm}[1][]{%
  \refstepcounter{algorithm}\par\medskip\noindent
  \begin{minipage}{\linewidth}\hrule\vspace{0.5ex}
  \textbf{Algorithm \thealgorithm\ #1}\par\small\vspace{0.5ex}%
}{%
  \vspace{0.5ex}\hrule\end{minipage}\par\medskip
}

\title{Spectra: Exact Component Transport for\\
Test-Time Prior Adaptation in\\
Simulation-Based Inference}

\ifpubliccopy
  \iclrfinalcopy
  \newcommand{\codeurl}{https://github.com/TtoastXin/Spectra}
  \newcommand{\aff}[2]{\par\hangindent=0.6em\hangafter=1\noindent$^{#1}$#2}
  \author{Xin Zhao$^{1,2}$\hspace{0.8em}Nico Scherf$^{3,4}$\hspace{0.8em}Robert Trampel$^{1}$\hspace{0.8em}
  Kerrin J. Pine$^{1}$\hspace{0.8em}Nikolaus Weiskopf$^{1,5,6}$\\[4pt]
  \parbox[t]{\dimexpr\textwidth-2\tabcolsep\relax}{\normalfont\fontsize{8}{9.6}\selectfont\raggedright
  \aff{1}{Department of Neurophysics, Max Planck Institute for Human Cognitive and Brain Sciences, Leipzig, Germany}
  \aff{2}{International Max Planck Research School on Cognitive NeuroImaging, Leipzig, Germany}
  \aff{3}{Methods and Development Group Neural Data Science and Statistical Computing, Max Planck Institute for Human Cognitive and Brain Sciences, Leipzig, Germany}
  \aff{4}{Center for Scalable Data Analytics and Artificial Intelligence (ScaDS.AI), Dresden/Leipzig, Germany}
  \aff{5}{Felix Bloch Institute for Solid State Physics, Faculty of Physics and Earth System Sciences, Leipzig University, Leipzig, Germany}
  \aff{6}{Functional Imaging Laboratory, Department of Imaging Neuroscience, UCL Queen Square Institute of Neurology, University College London, London, UK}
  \endgraf\vspace{4pt}\noindent\texttt{zhaox@cbs.mpg.de}}}
\else
  \author{Anonymous authors\\Paper under double-blind review}
\fi

\begin{document}
\maketitle
\ifpubliccopy\lhead{Preprint. Under review.}\fi

\begin{abstract}
Simulation-based inference (SBI) has become a powerful approach to Bayesian
inference in complex scientific models whose likelihoods are difficult or
impossible to evaluate. Amortized SBI learns reusable inference models from
simulated data, enabling rapid posterior inference for new observations, and
modern generative models have made these models increasingly expressive.
However, this reuse is limited to the prior distribution chosen during training, whereas scientific analyses often need
revised priors as knowledge accumulates or alternative assumptions are tested.
We introduce Spectra, a test-time adaptation method for diffusion-based SBI.
Spectra uses an exact score-transport identity to obtain the adapted
score from a frozen diffusion model in closed form for structured prior
changes, without additional simulation or training.
Across six SBI benchmarks, Spectra achieves accurate adaptation under strong
prior shifts at low online sampling cost. This enables pretrained SBI models to
incorporate updated prior information at test time.
\end{abstract}

\section{Introduction}
\label{sec:intro}

Bayesian inference estimates unknown model parameters $\theta$ from observed data $x$ 
and quantifies their uncertainty through the posterior $p(\theta\mid x)$. For complex scientific models, however,
the likelihood $p(x\mid\theta)$ is intractable or too costly to evaluate,
even though synthetic data can be generated readily from a simulator
\citep{cranmer2020frontier}.
This makes standard likelihood-based Markov chain
Monte Carlo impractical or unavailable \citep{robert2004montecarlo}. 
Simulation-based inference (SBI) instead learns from simulated
parameter--observation pairs \citep{cranmer2020frontier,lueckmann2021benchmark}. 
Amortized SBI makes this posterior inference reusable across observations by training an 
inference model once and applying it to new observations without 
repeating simulation and training \citep{greenberg2019apt,gloeckler2024simformer}. 
This reuse has enabled applications ranging from real-time gravitational-wave inference
\citep{dax2021gravitational,dax2025realtime} to cosmological constraints from
galaxy clustering \citep{hahn2024simbig} and identification of neural circuit
configurations from electrophysiological recordings
\citep{goncalves2020mechanistic}.

\looseness=-1
Posterior distributions in scientific applications are often multimodal or
strongly correlated, motivating flexible generative posterior models.
Diffusion and score-based models generate samples by learning to reverse a
gradual noising process
\citep{sohldickstein2015deep,ho2020ddpm,song2021sde},
and have been adapted to posterior sampling in SBI
\citep{geffner2023compositional,sharrock2024snse,gloeckler2024simformer}.
Flow matching and consistency models have also been developed for SBI: 
flow matching learns time-dependent velocity fields for posterior sampling, 
while consistency models reduce generation to only a few inference steps
\citep{wildberger2023fmpe,schmitt2024consistency}. 
Transformer-based approaches can amortize across datasets by conditioning on
a dataset provided as context at inference time, allowing a single model to be
reused across datasets and tasks
\citep{muller2022bayesian,chang2025ace,mittal2025incontext,whittle2025distribution}.
In the standard fixed-prior setting considered here, training pairs are generated by drawing $\theta$ from a chosen prior $\pi_{\rm tr}$ 
and then simulating $x\sim p(x\mid\theta)$. This amortization therefore applies across observations, but only under the training prior $\pi_{\rm tr}$.

However, the prior fixed at training time is rarely the only one needed in a
scientific analysis. In gravitational-wave astronomy, population-level
observations can inform priors for individual events
\citep{moore2021population}, while assumptions about black-hole masses and
spins can change the inferred properties of a source
\citep{vitale2017priors}. The same issue appears in other fields:
fossil-calibration priors affect estimates of species divergence times in
evolutionary biology \citep{warnock2015calibration}, and prior-sensitivity
analyses in cognitive science examine how evidence for competing models varies
with the prior \citep{elsemuller2023sensitivity}. Priors may therefore need to
change both as new information becomes available and as researchers test the
robustness of their conclusions to alternative assumptions
\citep{gelman2020workflow}. A single study may consequently require posterior
inference under several plausible priors.

Each of these posteriors could be computed separately using a non-amortized
inference procedure or by retraining the amortized model, but both options
forgo much of the benefit of amortization. Retraining, in particular, repeats
costly simulation and optimization for every new prior. Importance sampling
can instead reweight samples from the original posterior, but becomes
inefficient when the new posterior concentrates in regions poorly covered by
the original samples \citep{agapiou2017importance}. 
Another way to avoid per-prior retraining is to train across a family of priors, 
but test-time flexibility is constrained by the prior representation 
built into the model during training \citep{elsemuller2023sensitivity,chang2025ace,whittle2025distribution}.
In-context SBI with tabular foundation models avoids task-specific retraining,
but still requires suitable simulated examples as context at inference time
\citep{hollmann2025tabpfn,vetter2025effortless}. For an already trained
diffusion-SBI model, recent methods adapt the prior at test time, either through
an approximate sampling correction \citep{yang2026priorguide} or by training
an auxiliary correction model for each target prior \citep{zang2026dre}.
For this setting, an effective adapter must remain accurate under substantial
prior shifts, add little inference-time cost, 
and require neither additional simulation nor retraining.

\begin{figure}[t]
\centering
\includegraphics{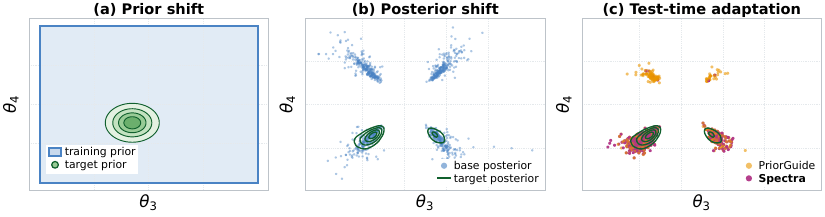}
\caption{\textbf{Prior shift and test-time adaptation on the simple likelihood complex posterior (SLCP) task.}
\textbf{(a)} A concentrated target prior replaces the broad training prior.
\textbf{(b)} The frozen model represents the base posterior under the training prior,
while the target prior concentrates posterior mass on two of its four modes.
\textbf{(c)} Both PriorGuide and Spectra recover the target modes.
However, PriorGuide retains visible mass outside the target contours,
whereas Spectra follows the target geometry more closely.
The posterior panels show the $(\theta_3,\theta_4)$ projection. Full results
across tasks and observations are reported in Table~\ref{tab:k1-results} and
Figure~\ref{fig:posterior-geometry-practical}.}
\label{fig:prior-adaptation}
\end{figure}

We introduce Spectra, a test-time method for adapting pretrained diffusion-SBI
models while leaving both the simulator and the trained model unchanged.
Because diffusion sampling repeatedly queries the score network across noise
levels, prior adaptation requires the corresponding score under the revised
prior. Spectra obtains this score by exploiting analytic structure in the known
prior change, using a single transformed query to the frozen model. For
isotropic exponential--quadratic prior-ratio factors, this transformation is
exact. Positive mixtures of such factors extend the construction to richer
prior changes. Figure~\ref{fig:prior-adaptation} illustrates the adaptation
problem and a representative result on one benchmark task. Our main contributions are:

\begin{itemize}

\item \textbf{Exact score transport and characterization.}
For prior changes that reweight the training prior by an 
isotropic exponential--quadratic factor, 
the adapted diffusion score follows exactly from one 
frozen-model query at a transformed state and noise level, 
without a score Jacobian or any Gaussian assumption on the 
base posterior. We further prove that, within this 
affine Gaussian transport construction, 
isotropic exponential--quadratic factors are 
the only positive reweightings that admit such an exact mapping.

\item \textbf{Composition of richer prior changes.}
Richer prior changes can be represented as positive mixtures of transportable
factors, yielding mixtures of exactly transported posterior components. For a
fixed observation, Spectra estimates the component weights once and samples one
component per reverse trajectory. Online sampling therefore still requires only
one frozen-score evaluation per reverse step, independent of the number of
components.

\item \textbf{Accuracy and efficiency.} Across six SBI benchmarks, Spectra
remains accurate under strong prior shifts and for mixture priors while keeping
one frozen-score query per reverse step. Compared with a strengthened version of
PriorGuide, the closest test-time method, its accuracy gains are largest on
nonlinear problems, and its online sampling is 2.4--3.2$\times$ faster.

\end{itemize}

\section{Background}
\label{sec:background}

\subsection{Simulation-based inference}
\label{sec:background-sbi}

Let $\theta\in\R^d$ denote the unknown parameters of a scientific model and
$x$ the observed data. A prior describes plausible parameter values before
observing $x$, while the likelihood $p(x\mid\theta)$ describes how observations
are generated from a given $\theta$. We write $\pi_{\rm tr}$ for the prior used
to train the inference model. Bayes' rule gives
\begin{equation}
p(\theta\mid x)
=
\frac{p(x\mid\theta)\pi_{\rm tr}(\theta)}
{\int p(x\mid\vartheta)\pi_{\rm tr}(\vartheta)\,d\vartheta}.
\label{eq:background-bayes}
\end{equation}

The posterior $p(\theta\mid x)$ describes which parameter values remain
plausible after observing $x$. In SBI, this posterior is learned from simulated
parameter--observation pairs: training pairs are generated by drawing
$\theta\sim\pi_{\rm tr}$ and then $x\sim p(x\mid\theta)$
\citep{cranmer2020frontier}. In amortized SBI, these pairs train a reusable
model of $p(\theta\mid x)$, which can then be applied to new observations
without further simulation or retraining
\citep{greenberg2019apt,gloeckler2024simformer}.

\subsection{Diffusion models}
\label{sec:background-diffusion}

Diffusion models provide a flexible way to represent such posteriors by
gradually corrupting samples with Gaussian noise and learning how to reverse
this corruption \citep{sohldickstein2015deep,ho2020ddpm}. Let $z_0\sim p_0$
denote a clean sample and $z_t$ its noisy counterpart at diffusion time
$t\in[0,1]$. We use the variance-exploding (VE) formulation, with increasing
noise standard deviation $\sigma(t)$ and $\sigma(0)=0$. The forward transition
and resulting marginal density are
\begin{equation}
p_{t|0}(z_t\mid z_0)
=\mathcal N\!\left(z_t;z_0,\sigma(t)^2I\right),
\qquad
p_t(z_t)=\int p_{t|0}(z_t\mid z_0)p_0(z_0)\,dz_0.
\label{eq:background-forward}
\end{equation}

The reverse process depends on the score
$s(z_t,t)=\nabla_{z_t}\log p_t(z_t)$, the gradient of the log density at the
current noisy state. Intuitively, the score points toward regions of higher
probability and provides the direction used for denoising. For the VE process,
the score determines the reverse-time stochastic differential equation
\begin{equation}
dz_t=-2\sigma(t)\dot\sigma(t)\nabla_{z_t}\log p_t(z_t)\,dt
+\sqrt{2\sigma(t)\dot\sigma(t)}\,dW_t,
\label{eq:background-reverse-sde}
\end{equation}
integrated from $t=1$ toward $t=0$ \citep{song2021sde}. Here
$\dot\sigma(t)=d\sigma(t)/dt$, and $W_t$ denotes a Wiener process in reverse
time. Because the marginal score is unknown, a neural network $s_\psi(z_t,t)$ is
trained using denoising score matching (DSM)
\citep{hyvarinen2005score,vincent2011connection}:
\begin{equation}
\mathcal L_{\rm DSM}(\psi)
=\E_{t,z_0,z_t}\!\left[
\lambda(t)\left\|
s_\psi(z_t,t)-\nabla_{z_t}\log p_{t|0}(z_t\mid z_0)
\right\|_2^2\right],
\label{eq:background-dsm}
\end{equation}
where $z_0\sim p_0$, $z_t\sim p_{t|0}(\cdot\mid z_0)$, $t$ is sampled across
noise levels, and $\lambda(t)>0$ weights those levels. For Gaussian corruption,
the conditional score is available analytically as
$\nabla_{z_t}\log p_{t|0}(z_t\mid z_0)=-(z_t-z_0)/\sigma(t)^2$.
Minimizing this objective learns the marginal score \(s(z_t,t)\) in expectation \citep{vincent2011connection}, providing the score needed by the reverse process. 
Sampling starts from a Gaussian approximation at the highest noise level and 
integrates the reverse dynamics using the learned score \citep{song2021sde}.

\subsection{Diffusion-based amortized SBI}
\label{sec:background-diffusion-sbi}

In diffusion-based SBI, the clean variable is the parameter vector,
$z_0=\theta$, and the observation $x$ conditions the denoising model. Applying
the forward process from Section~\ref{sec:background-diffusion} to the posterior
$p(\theta\mid x)$ gives the noisy posterior and its score
\begin{equation}
p_t(z\mid x)
=
\int p_{t|0}(z\mid\theta)p(\theta\mid x)\,d\theta,
\qquad
s_p(z,t,x)
=
\nabla_z\log p_t(z\mid x).
\label{eq:base-noisy-score}
\end{equation}
We refer to $p(\theta\mid x)$ as the base posterior and $s_p$ as its noisy
score. The subscript $p$ distinguishes these quantities from the adapted
posterior introduced below. A conditional network $s_\psi(z_t,t,x)$ is trained
with the same denoising score-matching objective as in
Section~\ref{sec:background-diffusion}, using simulated pairs
$\theta\sim\pi_{\rm tr}$ and $x\sim p(x\mid\theta)$, with $x$ supplied as an
additional condition
\citep{geffner2023compositional,gloeckler2024simformer}. After training, we freeze the network and denote its learned score estimate by
$\widehat s_p$. At inference time, the observation $x$ is fixed and
$\widehat s_p$ drives the reverse diffusion process, producing approximate
samples from the base posterior. The same model can be reused for new
observations, but its learned score still corresponds to the training prior
$\pi_{\rm tr}$.

\section{Spectra: Exact component transport}
\label{sec:method}

\subsection{Prior adaptation as score correction}
\label{sec:prior-adaptation}

Suppose the same observation $x$ is analyzed under a new prior
$\pi_{\rm new}$, while the simulator remains unchanged. We assume
$\pi_{\rm new}$ places no mass outside the support of $\pi_{\rm tr}$ and define
the prior ratio $r(\theta)=\pi_{\rm new}(\theta)/\pi_{\rm tr}(\theta)$. Since the likelihood is unchanged, Bayes' rule gives
\begin{equation}
q(\theta\mid x)
=
\frac{p(\theta\mid x)r(\theta)}{Z_r(x)},
\qquad
Z_r(x)=\E_{p(\theta\mid x)}[r(\theta)].
\label{eq:background-clean-reweight}
\end{equation}
In the diffusion formulation,
Eq.~\eqref{eq:background-clean-reweight} determines the target distribution at
the clean endpoint. However, reverse sampling requires the corresponding target
score at every noise level.
A Bayes-rule calculation (Appendix~\ref{app:noisy-reweight-derivation}) gives
\begin{equation}
s_q(z,t,x)
=
s_p(z,t,x)
+
\nabla_z\log
\E_{p(\theta\mid z,x)}[r(\theta)].
\label{eq:prior-score-correction}
\end{equation}
Here \(p(\theta\mid z,x)\) is the conditional distribution of the clean parameter given the noisy state \(z\) under the base posterior. 
The frozen model provides \(s_p\), but not \(p(\theta\mid z,x)\), so the correction in Eq.~\eqref{eq:prior-score-correction} 
is not directly available. PriorGuide makes this correction tractable by approximating \(p(\theta\mid z,x)\) 
as Gaussian \citep{yang2026priorguide}, while density-ratio adaptation learns an auxiliary correction for each target prior \citep{zang2026dre}. 
Spectra instead shows that, for structured prior ratios, 
the adapted score can be obtained exactly from a transformed query 
to the frozen base score.

\subsection{Exact single-factor transport}
\label{sec:single-factor}

Write $\tau=\sigma(t)^2$ for the noise variance, and index scores by $\tau$ in what follows. Consider the isotropic exponential--quadratic factor
\begin{equation}
\phi_{a,\kappa,c}(\theta)
=
\exp\!\left(
c+a^\top\theta-\frac{\kappa}{2}\lVert\theta\rVert^2
\right),
\label{eq:transport-factor}
\end{equation}
where $a\in\R^d$ and $\kappa,c\in\R$. If
$r(\theta)\propto\phi(\theta)$, then
$q^\phi(\theta\mid x)\propto p(\theta\mid x)\phi(\theta)$ is the target
posterior. The quadratic term is isotropic: a single scalar $\kappa$ controls its strength
in every direction, while $a$ controls an arbitrary linear tilt. For $\kappa>0$,
$\phi$ is proportional to an isotropic Gaussian with mean $a/\kappa$ and
covariance $\kappa^{-1}I$. For $\kappa=0$, it reduces to a linear exponential
tilt. The constant $c$ only rescales the factor and cancels after posterior
normalization. This restriction is on the prior-ratio factor while the base
posterior itself can be arbitrary. For Gaussian training and target priors, 
their means may differ arbitrarily, 
while their precision matrices may differ only by an isotropic term.

The key observation is that multiplying the Gaussian corruption kernel by
$\phi$ produces another Gaussian kernel at a transformed state and noise level:
\begin{equation}
\phi(\theta)\mathcal N(z;\theta,\tau I)
=
C_\tau(z)
\mathcal N\!\left(m_\tau(z);\theta,\rho_\tau I\right),
\label{eq:kernel-factor-main}
\end{equation}
where $C_\tau(z)>0$ is an analytic prefactor and
\begin{equation}
D_\tau=1+\kappa\tau,\qquad
\rho_\tau=\frac{\tau}{D_\tau},\qquad
m_\tau(z)=\frac{z+\tau a}{D_\tau}.
\label{eq:transform-pieces}
\end{equation}
We require $D_\tau>0$ at every noise level used by the sampler. This holds
automatically for $\kappa\geq0$; for $\kappa<0$, it restricts the schedule to
$\tau<-1/\kappa$. Substituting this identity into the integral defining the
noisy posterior and differentiating with respect to $z$ gives the target score
in closed form (Appendix~\ref{app:component-proof}):
\begin{equation}
\boxed{
s_{q^\phi}(z,\tau,x)
=
\frac{a-\kappa z}{D_\tau}
+
\frac{1}{D_\tau}
s_p\!\left(m_\tau(z),\rho_\tau,x\right).
}
\label{eq:exact-score}
\end{equation}
\begin{proposition}[Exact score transport]
\label{prop:transport}
Let $p(\theta\mid x)$ be any base posterior density with a differentiable
Gaussian smoothing, and let $q^\phi\propto p\phi$ be normalizable.
For $\tau>0$ with $D_\tau>0$,
Eq.~\eqref{eq:exact-score} gives the exact noisy score of $q^\phi$.
\end{proposition}
Each target-score evaluation therefore requires one base-score query at
$(m_\tau(z),\rho_\tau)$ and a closed-form correction. The result applies to arbitrary base-posterior geometry under the stated
regularity conditions. Appendix~\ref{app:component-proof} gives the full derivation, learned-score
error identity, and special cases.

\subsection{Characterization of exact transport}

The transport above follows from a kernel identity that is independent of the
base posterior. This raises a natural question: which positive factors admit
the same form of exact transport?

At a fixed noise level $\tau>0$, consider
\begin{equation}
\phi(\theta)\mathcal N(z;\theta,\tau I)
=
C_\tau(z)
\mathcal N(A_\tau z+u_\tau;\theta,\rho_\tau I),
\label{eq:characterization-kernel}
\end{equation}
where $\rho_\tau>0$, $A_\tau\in\R^{d\times d}$, and $u_\tau\in\R^d$
depend only on $\tau$, and $C_\tau(z)>0$.

\begin{theorem}[Characterization of exact transport]
\label{thm:characterization}
Fix $\tau>0$. Suppose a positive factor $\phi$ admits the kernel transport in
Eq.~\eqref{eq:characterization-kernel} for some affine state transform
$A_\tau z+u_\tau$ and scalar variance $\rho_\tau>0$.
Then $\phi$ must have the isotropic exponential--quadratic form in
Eq.~\eqref{eq:transport-factor}, and necessarily
$A_\tau=I/D_\tau$, $u_\tau=\rho_\tau a$, and
$\rho_\tau=\tau/D_\tau$.
Conversely, every factor of the form in
Eq.~\eqref{eq:transport-factor} satisfies
Eq.~\eqref{eq:characterization-kernel} whenever $D_\tau>0$.
\end{theorem}
Thus this family is not chosen merely for convenience: within affine state
transforms and scalar Gaussian variances, it is the complete class of positive
factors admitting exact transport. Appendix~\ref{app:closure} gives the proof. 
For the Gaussian-factor special case, closely related exact transformed-query
identities arise in Gaussian conditioning
\citep{holderrieth2026glass,mammadov2026eps}. Appendix~\ref{sec:related-work}
discusses the relation.

\subsection{Combining posterior components}
\label{sec:mixture}

These exact single-factor transports can be used as building blocks for
richer prior ratios. Specifically, consider a positive weighted sum of $K$
factors
\begin{equation}
r(\theta)=\sum_{k=1}^K b_k\phi_k(\theta),
\qquad b_k>0,
\label{eq:positive-factor-representation}
\end{equation}
where each $\phi_k$ has the form in Eq.~\eqref{eq:transport-factor}. Each
factor defines a posterior component
$q_k(\theta\mid x)=p(\theta\mid x)\phi_k(\theta)/Z_k(x)$ with
$Z_k(x)=\E_{p(\theta\mid x)}[\phi_k(\theta)]$,
assuming $0<Z_k(x)<\infty$. The score of each component is given exactly by Eq.~\eqref{eq:exact-score}.
Substituting Eq.~\eqref{eq:positive-factor-representation} into
Eq.~\eqref{eq:background-clean-reweight} and normalizing gives
\begin{equation}
q(\theta\mid x)
=
\sum_{k=1}^K \alpha_k(x)q_k(\theta\mid x),
\qquad
\alpha_k(x)
=
\frac{b_k Z_k(x)}
{\sum_j b_j Z_j(x)}.
\label{eq:posterior-mixture}
\end{equation}
Thus the coefficients $b_k$ belong to the chosen representation of the known
prior ratio and are fixed independently of $x$. In contrast, the posterior
mixture weights $\alpha_k(x)$ depend on the observation through $Z_k(x)$.
Appendix~\ref{app:mixture-proof} gives the mixture algebra. If
Eq.~\eqref{eq:positive-factor-representation} approximates $r$ rather than
representing it exactly, the same algebra gives the posterior induced by the
approximating ratio. Appendix~\ref{app:error-proof} quantifies its discrepancy
from the exact target posterior.

\subsection{Sampling one component per trajectory}

Eq.~\eqref{eq:posterior-mixture} expresses the target posterior as a mixture
of the component posteriors $q_k$. For a fixed observation $x$,
$\alpha_k(x)$ are the static weights of this posterior mixture. Sampling from the mixture is straightforward: 
for each output sample, first
select component $k$ with probability $\alpha_k(x)$, then draw from
$q_k(\theta\mid x)$. Let $J$ denote the selected component:
\begin{equation}
\Pr(J=k\mid x)=\alpha_k(x),
\qquad
\theta\mid J=k,x\sim q_k(\theta\mid x).
\label{eq:static-component-sampling}
\end{equation}
Marginalizing over $J$ recovers the target mixture in
Eq.~\eqref{eq:posterior-mixture}. Spectra realizes the conditional draw by
reverse diffusion, using component $J$'s transported score throughout the
trajectory. Each reverse step therefore requires one base-score query,
independent of the number of components $K$.

A natural alternative is to follow the score of the full mixture, which
requires state-dependent component responsibilities at every reverse step.
These responsibilities depend on normalized component densities and are not
directly available from local queries to the frozen score model. Sampling the
component once avoids this per-step recombination. Appendix~\ref{app:mixture-proof}
gives the formal argument and derives the full-mixture score. Implementation
details are given in Appendix~\ref{app:sampling-settings}.

\subsection{Estimating component weights}
\label{sec:weights}

For $K>1$, the only remaining quantities to estimate in
Eq.~\eqref{eq:posterior-mixture} are the observation-specific normalizers
$Z_k(x)=\E_{p(\theta\mid x)}[\phi_k(\theta)]$.
In practice, the estimator operates on the finite reverse chain of the base
sampler. Let $\widetilde Z_k(x)$ denote the corresponding finite-chain
counterpart of $Z_k(x)$. A direct estimate of $\widetilde Z_k(x)$ averages $\phi_k$ over endpoint
samples from the base chain. However, this can suffer from poor overlap when
component $k$ concentrates in a region rarely visited by those samples.
We therefore use the corresponding transported component process as an
importance-sampling proposal.

Let $\omega=(z_R,\ldots,z_0)$ denote a reverse trajectory, and let
$P(\omega)$ and $Q_k(\omega)$ be the path densities of the base and
component-$k$ chains used by the estimator. Their finite-step transition
densities are known, so the trajectory likelihood ratio $P(\omega)/Q_k(\omega)$
is tractable even when the endpoint densities are not. Under common path
support, change of measure for the finite reverse chain gives
\begin{equation}
\widetilde Z_k(x)
=
\E_{\omega\sim P}[\phi_k(z_0)]
=
\E_{\omega\sim Q_k}\!\left[
\phi_k(z_0)\frac{P(\omega)}{Q_k(\omega)}
\right].
\label{eq:path-z}
\end{equation}
Monte Carlo averaging of the right-hand side gives $\widehat Z_k$, which is
used in Eq.~\eqref{eq:posterior-mixture} to form the component weights.
Appendix~\ref{app:path} gives the trajectory likelihood ratio, overlap
diagnostics, the direct estimator, and the full procedure.
Appendix~\ref{app:error-proof} separates finite-sampler and Monte Carlo
weight-estimation errors. This is an upfront, observation-specific calculation, with proposal
trajectories generated separately for each component. For a fixed observation
and factor set, the estimated normalizers can be reused when only the
coefficients $b_k$ change.

\section{Experiments}
\label{sec:experiments}

We evaluate Spectra on six SBI benchmarks under both single-factor and
mixture prior shifts, adapting frozen diffusion-based Simformer models
\citep{gloeckler2024simformer} without retraining or additional simulator
calls. Across these experiments, compared with a strengthened version of
PriorGuide, Spectra improves posterior accuracy most clearly on nonlinear
tasks under strong shifts, remains comparable on OUP and the Gaussian Linear
controls, improves mixture-posterior agreement on all six tasks, and reduces
steady-state single-factor online sampling time by $2.4$--$3.2\times$ at the
sampler settings used for accuracy.

\subsection{Experimental setup}

We consider six benchmark tasks: Two Moons and simple likelihood complex posterior (SLCP) \citep{lueckmann2021benchmark}, 
the Ornstein--Uhlenbeck process (OUP) \citep{uhlenbeck1930brownian}, 
Bayesian causal inference (BCI) in multisensory perception \citep{kording2007causal}, 
and Gaussian Linear problems \citep{lueckmann2021benchmark} in 10 and 20 dimensions (GL-10D and GL-20D), 
included as Gaussian controls. 
Following \citet{yang2026priorguide}, single-factor target priors are isotropic Gaussians with $0.5$ (mild) or $0.2$ (strong) times 
the training-prior marginal standard deviation. Mixture shifts use two strong Gaussian components. 
Full prior construction, posterior references, and support checks are given in Appendix~\ref{app:experiment-details}.
These benchmark shifts lie in the exactly transportable component class.
Appendix~\ref{app:component-approx-results} studies anisotropic and correlated
Gaussian targets outside the exact single-factor class using positive-mixture
approximations, where Spectra remains more accurate than PriorGuide.

Baselines include PriorGuide \citep{yang2026priorguide}, 
the unadapted frozen model (Base), and sampling-importance-resampling (SIR) \citep{rubin1988sir}. 
PriorGuide receives the exact analytic prior ratio and uses the same posterior-sampling settings as Spectra, 
while SIR reweights samples from Base by the prior ratio. 
We write a sampler configuration as $(N,L)$, 
where $N$ is the number of reverse time-grid points and $L$ the number of Langevin corrector updates per interval. 
Sampling and weight-estimation details are given in Appendix~\ref{app:sampling-settings} and Appendix~\ref{app:path}.

Posterior agreement is assessed using two complementary metrics. The classifier two-sample test \citep[C2ST;][]{lopezpaz2017c2st} 
assesses joint agreement, including parameter dependencies, with an ideal value of $0.5$ corresponding to chance-level classification. 
Mean marginal total variation (MMTV) measures average discrepancy across one-dimensional marginals, with an ideal value of zero.

\begin{table}[!t]
\centering
\caption{\textbf{Posterior accuracy under single-factor and mixture prior shifts.}
Mean (standard deviation) across ten observations or 30 prior--observation
pairs. Ideal values are $0.5$ for C2ST and $0$ for MMTV. Bold marks the mean
closest to the ideal value within each task and setting.}
\label{tab:k1-results}
\label{tab:k2-results}
\begingroup
\fontsize{7}{8.2}\selectfont
\setlength{\tabcolsep}{2pt}
\renewcommand{\arraystretch}{1.0}
\begin{tabular*}{\linewidth}{@{\extracolsep{\fill}}llrrrrrr@{}}
\toprule
 & & \multicolumn{2}{c}{Single-factor, mild} & \multicolumn{2}{c}{Single-factor, strong} & \multicolumn{2}{c}{Mixture} \\
\cmidrule(lr){3-4}\cmidrule(lr){5-6}\cmidrule(l){7-8}
Task & Method & C2ST & MMTV & C2ST & MMTV & C2ST & MMTV \\
\midrule
\multirow{4}{*}{Two Moons} & Base & \benchval{0.568}{0.036} & \benchval{0.219}{0.087} & \benchval{0.691}{0.095} & \benchval{0.426}{0.192} & \benchval{0.736}{0.074} & \benchval{0.510}{0.147} \\
 & SIR & \benchval{0.525}{0.005} & \benchval{\textbf{0.075}}{0.011} & \benchval{0.544}{0.018} & \benchval{\textbf{0.070}}{0.012} & \benchval{0.543}{0.018} & \benchval{\textbf{0.062}}{0.009} \\
 & PriorGuide & \benchval{0.528}{0.012} & \benchval{0.123}{0.046} & \benchval{0.556}{0.034} & \benchval{0.165}{0.086} & \benchval{0.561}{0.036} & \benchval{0.176}{0.086} \\
 & Spectra & \benchval{\textbf{0.521}}{0.005} & \benchval{0.093}{0.016} & \benchval{\textbf{0.532}}{0.015} & \benchval{0.102}{0.030} & \benchval{\textbf{0.521}}{0.009} & \benchval{0.073}{0.014} \\
\midrule
\multirow{4}{*}{SLCP} & Base & \benchval{0.765}{0.057} & \benchval{0.255}{0.041} & \benchval{0.917}{0.019} & \benchval{0.466}{0.072} & \benchval{0.946}{0.029} & \benchval{0.498}{0.075} \\
 & SIR & \benchval{0.640}{0.086} & \benchval{\textbf{0.123}}{0.054} & \benchval{0.713}{0.067} & \benchval{\textbf{0.101}}{0.043} & \benchval{0.796}{0.086} & \benchval{0.205}{0.269} \\
 & PriorGuide & \benchval{0.632}{0.108} & \benchval{0.128}{0.061} & \benchval{0.637}{0.061} & \benchval{0.122}{0.034} & \benchval{0.618}{0.054} & \benchval{0.150}{0.057} \\
 & Spectra & \benchval{\textbf{0.631}}{0.108} & \benchval{0.128}{0.060} & \benchval{\textbf{0.608}}{0.073} & \benchval{0.107}{0.044} & \benchval{\textbf{0.570}}{0.056} & \benchval{\textbf{0.081}}{0.031} \\
\midrule
\multirow{4}{*}{OUP} & Base & \benchval{0.587}{0.064} & \benchval{0.186}{0.075} & \benchval{0.742}{0.088} & \benchval{0.397}{0.079} & \benchval{0.703}{0.063} & \benchval{0.322}{0.085} \\
 & SIR & \benchval{0.524}{0.009} & \benchval{0.054}{0.009} & \benchval{0.548}{0.050} & \benchval{\textbf{0.045}}{0.011} & \benchval{0.529}{0.011} & \benchval{\textbf{0.046}}{0.012} \\
 & PriorGuide & \benchval{\textbf{0.507}}{0.003} & \benchval{\textbf{0.053}}{0.006} & \benchval{\textbf{0.505}}{0.009} & \benchval{0.047}{0.011} & \benchval{0.509}{0.008} & \benchval{0.049}{0.017} \\
 & Spectra & \benchval{0.508}{0.002} & \benchval{\textbf{0.053}}{0.006} & \benchval{\textbf{0.505}}{0.009} & \benchval{\textbf{0.045}}{0.010} & \benchval{\textbf{0.507}}{0.006} & \benchval{\textbf{0.046}}{0.012} \\
\midrule
\multirow{4}{*}{GL-10D} & Base & \benchval{0.898}{0.024} & \benchval{0.308}{0.028} & \benchval{0.997}{0.002} & \benchval{0.664}{0.038} & \benchval{0.996}{0.003} & \benchval{0.636}{0.044} \\
 & SIR & \benchval{0.760}{0.056} & \benchval{0.225}{0.260} & \benchval{0.998}{0.002} & \benchval{0.965}{0.064} & \benchval{0.999}{0.001} & \benchval{0.965}{0.103}\rlap{\textsuperscript{$\dagger$}} \\
 & PriorGuide & \benchval{0.541}{0.020} & \benchval{\textbf{0.044}}{0.007} & \benchval{0.526}{0.019} & \benchval{\textbf{0.038}}{0.006} & \benchval{0.557}{0.029} & \benchval{0.080}{0.028} \\
 & Spectra & \benchval{\textbf{0.540}}{0.019} & \benchval{\textbf{0.044}}{0.007} & \benchval{\textbf{0.524}}{0.017} & \benchval{\textbf{0.038}}{0.005} & \benchval{\textbf{0.543}}{0.030} & \benchval{\textbf{0.052}}{0.018} \\
\midrule
\multirow{4}{*}{GL-20D} & Base & \benchval{0.947}{0.016} & \benchval{0.300}{0.017} & \benchval{1.000}{0.000} & \benchval{0.645}{0.017} & \benchval{0.999}{0.001} & \benchval{0.640}{0.030} \\
 & SIR & \benchval{0.927}{0.029} & \benchval{0.888}{0.061} & \benchval{1.000}{0.000} & \benchval{0.852}{0.139}\rlap{\textsuperscript{$\dagger$}} & \benchval{1.000}{0.000} & \benchval{0.900}{0.125}\rlap{\textsuperscript{$\dagger$}} \\
 & PriorGuide & \benchval{\textbf{0.555}}{0.022} & \benchval{\textbf{0.042}}{0.003} & \benchval{\textbf{0.529}}{0.021} & \benchval{0.035}{0.004} & \benchval{0.529}{0.011} & \benchval{0.051}{0.021} \\
 & Spectra & \benchval{\textbf{0.555}}{0.023} & \benchval{\textbf{0.042}}{0.003} & \benchval{\textbf{0.529}}{0.020} & \benchval{\textbf{0.034}}{0.004} & \benchval{\textbf{0.521}}{0.012} & \benchval{\textbf{0.036}}{0.016} \\
\midrule
\multirow{4}{*}{BCI} & Base & \benchval{0.758}{0.030} & \benchval{0.254}{0.049} & \benchval{0.948}{0.016} & \benchval{0.497}{0.045} & \benchval{0.946}{0.015} & \benchval{0.485}{0.039} \\
 & SIR & \benchval{0.591}{0.026} & \benchval{0.086}{0.031} & \benchval{0.799}{0.066} & \benchval{0.224}{0.312} & \benchval{0.789}{0.061} & \benchval{0.183}{0.206} \\
 & PriorGuide & \benchval{0.569}{0.028} & \benchval{0.087}{0.028} & \benchval{0.575}{0.013} & \benchval{0.081}{0.011} & \benchval{0.548}{0.019} & \benchval{0.072}{0.020} \\
 & Spectra & \benchval{\textbf{0.567}}{0.028} & \benchval{\textbf{0.085}}{0.028} & \benchval{\textbf{0.551}}{0.015} & \benchval{\textbf{0.066}}{0.010} & \benchval{\textbf{0.544}}{0.017} & \benchval{\textbf{0.061}}{0.011} \\
\bottomrule
\end{tabular*}

\par\vspace{4pt}
\raggedright\fontsize{7.5}{8.8}\selectfont
PriorGuide receives the exact prior ratio. Sampling configurations are $(N,L)=(25,8)$ for PriorGuide and Spectra,
except single-factor BCI, where both use $(250,0)$ following
\citet{yang2026priorguide}; Base and SIR use $(100,0)$. At the matched
$(100,0)$ configuration, Spectra attains a lower MMTV than SIR on Two Moons
and SLCP (Appendices~\ref{app:k1-controlled} and~\ref{app:k2-systematic}).
\par
\textsuperscript{$\dagger$}\ SIR MMTV is undefined for some Gaussian Linear
runs. Valid counts are reported in Appendices~\ref{app:k1-full}
and~\ref{app:k2-full}.
\par
\endgroup
\end{table}

\subsection{Single-factor prior adaptation}
\label{sec:k1-exp}

Under strong prior shifts, Spectra shows better agreement with the reference
posterior than PriorGuide on Two Moons, SLCP, and BCI, while the two methods
perform similarly on OUP and the Gaussian Linear controls
(Table~\ref{tab:k1-results}).

\looseness=-1
To separate shift strength from changes in prior center and observation, we
additionally vary the target-prior width while holding the prior center,
observations, backbone, and sampler fixed. The gap between PriorGuide and
Spectra grows systematically on Two Moons as the target prior narrows, but
remains essentially flat on OUP
(Figure~\ref{fig:severity-sweep}; Appendix~\ref{app:severity-sweep}).
The effect of shift severity therefore depends on posterior geometry 
rather than shift strength alone.

Figure~\ref{fig:posterior-geometry-practical} visualizes representative
strong-shift posteriors. Spectra follows the target geometry closely across
the displayed regimes, whereas PriorGuide exhibits off-target mass on the
nonlinear examples. SIR illustrates a complementary trade-off: reweighting
can be effective when the Base posterior already has substantial overlap with
the target, but becomes inefficient as that overlap deteriorates. Additional
controlled comparisons and diagnostics are reported in
Appendix~\ref{app:k1-full}.

\begin{figure}[!ht]
\centering
\includegraphics[width=\linewidth]{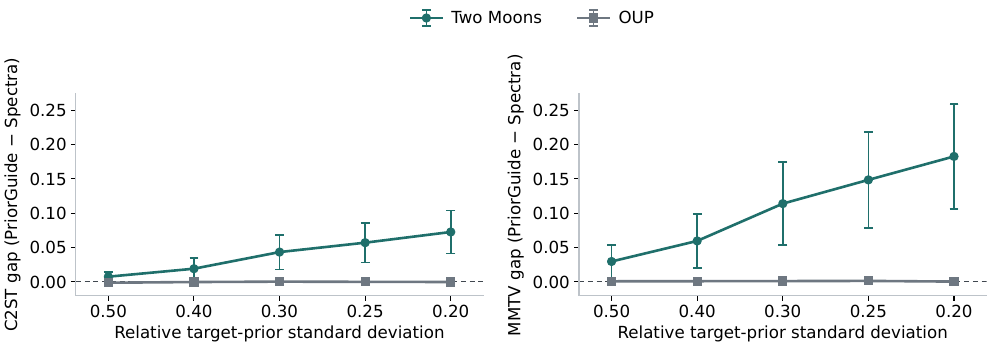}
\caption{\textbf{Accuracy gap as the target prior narrows.}
Mean paired PriorGuide--Spectra difference with 95\% paired $t$ confidence
intervals across ten observations. All experimental factors except
target-prior width are held fixed. Positive gaps indicate better performance
by Spectra.}
\label{fig:severity-sweep}
\end{figure}

\begin{figure}[!ht]
\centering
\includegraphics[width=\linewidth]{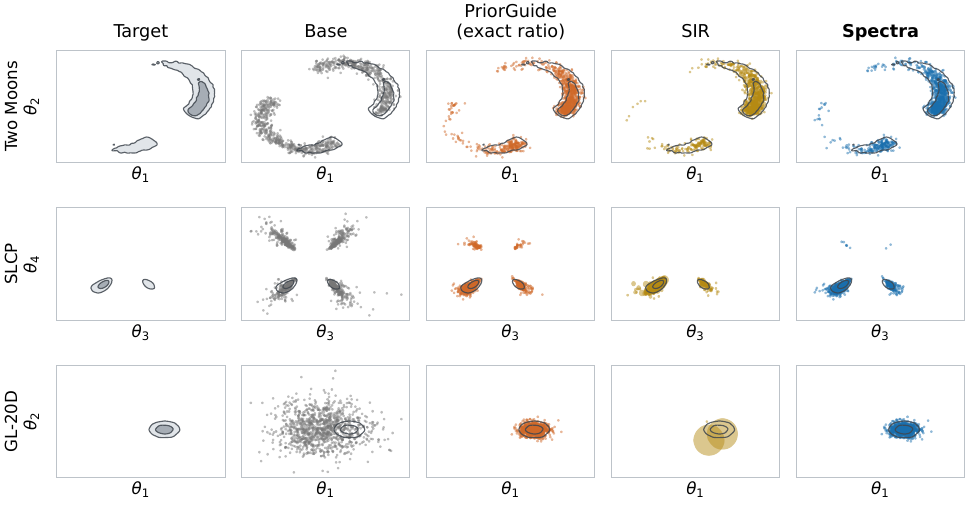}
\caption{\textbf{Posterior geometry under strong single-factor prior shifts.}
Reference contours enclose 50\% and 90\% posterior mass. Markers show 1,000
samples from each method for one fixed observation, backbone, and sampling
seed per task. Larger SIR markers indicate repeated resampling of the same
Base sample. SLCP and GL-20D show two-dimensional projections. 
Table~\ref{tab:k1-results} reports full-dimensional accuracy.}
\label{fig:posterior-geometry-practical}
\end{figure}

\subsection{Mixture prior shifts}
\label{sec:k2-exp}

Across 180 prior--observation pairs on six tasks, Spectra with path-space
weight estimation shows better agreement with the reference posterior than
PriorGuide on every task, with the clearest gains on Two Moons and SLCP
(Table~\ref{tab:k2-results}). For 162 of the 180 pairs, one component accounts
for more than 95\% of the posterior mass after observing $x$. In the remaining 18 pairs, both components retain at least 5\% posterior
mass, and the aggregate advantage over PriorGuide persists (Appendix~\ref{app:mixture-allocation}).

We also evaluate mixtures with more than two components on Two Moons
(Appendix~\ref{app:k-composition}). Posterior accuracy remains stable up to
$K=16$. At $K=16$, five components retain at least 5\% posterior mass for each
of the three observations, so the posterior remains genuinely multi-component.
For these larger mixtures, path-space weight estimates remain close to the
corresponding reference weights. Full mixture results and weight-estimation
diagnostics are reported in Appendix~\ref{app:k2-full}.

\begin{figure}[!ht]
\centering
\includegraphics[width=0.9\linewidth]{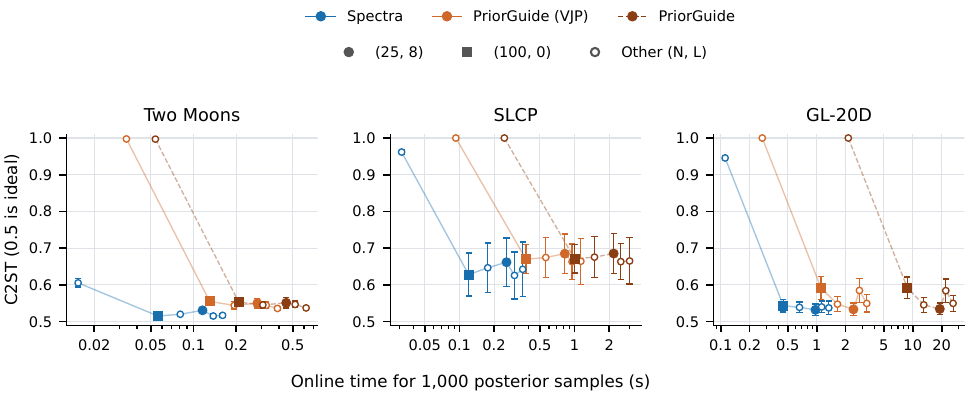}
\caption{\textbf{Single-factor accuracy versus online sampling time under
strong prior shifts.} PriorGuide (VJP) uses a vector--Jacobian product for
the same guidance update and preserves PriorGuide's accuracy
(Appendix~\ref{app:compute-details}). Timing covers the reverse sampling
stage for 1,000 posterior samples after compilation. One-time setup costs are
excluded.}
\label{fig:accuracy-compute}
\end{figure}

\subsection{Accuracy and computational cost}
\label{sec:compute-exp}

Spectra requires one transformed base-score evaluation per 
reverse step, whereas PriorGuide additionally differentiates through 
the frozen score model to compute its guidance. 
For timing, we implement PriorGuide's guidance with 
a vector--Jacobian product (VJP) instead of the full Jacobian 
used by the benchmarked implementation. 
Even with this faster implementation, Spectra reduces steady-state 
online sampling time by $2.4$--$3.2\times$ across five timing tasks
at the sampler settings of Table~\ref{tab:k1-results}.
Across the sampler budgets in Figure~\ref{fig:accuracy-compute}, increasing PriorGuide's sampling
effort does not close the accuracy--cost gap.
Detailed timing results and
implementation comparisons are reported in Appendix~\ref{app:compute-details}.

For mixture priors, fixing one component per trajectory keeps the subsequent
sampling cost independent of the number of components. On Two Moons, the
measured Spectra sampling time is unchanged from $K=1$ to $K=64$. The
$K$-dependent computation is instead concentrated in observation-specific
component-weight estimation, whose path-space cost grows linearly with $K$.
Weight estimation can dominate the initial cost for a new observation, while
additional samples for the same observation remain inexpensive. Detailed
timing and scaling results are reported in
Appendices~\ref{app:mixture-cost} and~\ref{app:k-scaling}.

\section{Discussion}
\label{sec:discussion}

Spectra allows amortized diffusion-based SBI models to be reused 
when prior assumptions change, without repeating expensive simulation or
training. It exploits structure in the known prior change and propagates 
its effect through the frozen score model, so a single pretrained model can support 
revised priors and prior-sensitivity analyses at test time. 
Composition extends the same principle to richer prior changes while 
keeping the per-step cost at one frozen-score evaluation.

\paragraph{Limitations.}
Spectra derives its exactness and low online sampling cost from structure in the prior
change rather than offering unrestricted prior flexibility. The present theory
assumes isotropic Gaussian corruption, and adaptation is limited to the support
of the training prior. Broader prior changes, such as anisotropic or correlated Gaussians, require
structured representations like the positive mixtures of Appendix~\ref{app:component-approx-results}. Exactness concerns the transported score rather than the
complete numerical sampler, which still inherits error from the learned score,
terminal initialization, and finite integration. For multi-component
representations, greater flexibility also introduces observation-specific
weight estimation and an additional setup cost.

\paragraph{Conclusion.}
\looseness=-1
Spectra shows that changing prior assumptions need not require 
changing the pretrained inference model itself. 
When the prior update has suitable structure, its effect can be 
propagated analytically through a frozen diffusion-SBI model, 
with composition extending the same idea to richer prior changes. 
More broadly, this suggests that structured test-time transformations can extend the useful 
life of expensive amortized inference models as the scientific assumptions 
they encode are revised.

\section*{Reproducibility statement}

Derivations, proofs, and error analyses are in 
Appendices~\ref{app:notation}--\ref{app:component-approx}; 
experimental protocols, reference validation, 
and full results are in Appendices~\ref{app:experiment-details}--\ref{app:component-approx-results}. 
\ifpubliccopy
The code at \url{\codeurl} includes the scripts, fixed random seeds, model and task configurations,
and experiment settings needed to rerun the reported experiments.
\fi Unless explicitly stated otherwise, 
headline comparisons use the exact analytic prior ratio when available and the stated matched posterior-sampling protocols.

\section*{AI use disclosure}

Generative AI tools were used as interactive assistants under author direction for research brainstorming and methodological feedback; 
drafting and checking parts of mathematical derivations and proofs; supporting implementation, debugging, 
and experimental analysis; and manuscript drafting, editing, and LaTeX preparation. 
The scientific direction, methodological and experimental decisions, validation of mathematical and empirical results, 
and final presentation were determined by the authors. AI-assisted outputs were treated as provisional and reviewed by the authors, 
who take responsibility for all claims, results, and conclusions in this work.

\ifpubliccopy
\section*{Acknowledgments}

Computations were performed on the HPC system Raven at the Max Planck Computing and Data Facility (MPCDF).
\fi

\bibliography{references}
\bibliographystyle{iclr2027_conference}

\appendix
\clearpage
\renewcommand{\FloatBarrier}{}
\renewcommand{\topfraction}{0.95}
\renewcommand{\bottomfraction}{0.7}
\renewcommand{\textfraction}{0.07}
\renewcommand{\floatpagefraction}{0.8}
\setcounter{topnumber}{4}
\setcounter{bottomnumber}{3}
\setcounter{totalnumber}{6}

\begin{center}
{\LARGE\scshape Appendix}
\end{center}
\medskip

The appendix is organized as follows.
\begin{itemize}
\item Appendix~\ref{sec:related-work} provides extended related work.
\item Appendices~\ref{app:notation}--\ref{app:mixture-proof} collect notation, derivations, and proofs.
\item Appendices~\ref{app:error-proof}--\ref{app:path} cover error bounds and normalizer estimation, and Appendix~\ref{app:component-approx} approximates correlated Gaussian priors by positive factors.
\item Appendices~\ref{app:experiment-details}--\ref{app:component-approx-results} show experimental protocols and full results.
\end{itemize}

\section{Extended related work}
\label{sec:related-work}

\paragraph{Prior adaptation in amortized inference.}
PriorGuide is the closest method to our setting. It adapts a frozen
diffusion-SBI model from a clean prior ratio using a Gaussian approximation
of the reverse conditional \citep{yang2026priorguide}, whereas Spectra
exploits prior-ratio structure to obtain the adapted score directly from the
base score. A contemporaneous approach instead learns the noisy ratio
correction with a target-prior-specific auxiliary network, avoiding new
simulator calls but requiring additional training \citep{zang2026dre}.
Other methods obtain prior flexibility by exposing the inference model to a
family of priors or model alternatives during training
\citep{chang2025ace,whittle2025distribution,elsemuller2023sensitivity}.
Automatic posterior transformation addresses proposal changes during
sequential simulation and retraining \citep{greenberg2019apt}, while
likelihood-based SBI permits a new prior to be imposed after training at the
cost of per-observation posterior sampling \citep{papamakarios2019snl}.
Related work also studies robustness of amortized inference through
self-consistency objectives on unlabeled observations
\citep{mishra2026selfconsistency}. Our setting instead starts from a
posterior model trained under a single prior and keeps it frozen during
adaptation.

\paragraph{Conditioning and composition in diffusion models.}
Several diffusion methods construct approximate conditional scores for
inverse problems, including pseudoinverse guidance, diffusion posterior
sampling, and Tweedie moment-projected methods
\citep{song2023pigdm,chung2023dps,boys2024tmpd}. Free Hunch estimates local
denoiser covariance using information from training data and the sampling
trajectory \citep{rissanen2025freehunch}. Exact Gaussian-conditioning
identities can also lead to transformed model queries: GLASS combines
Gaussian observations through a transformed input and noise level
\citep{holderrieth2026glass}, while EPS derives an exact posterior-score
identity for linear Gaussian measurements \citep{mammadov2026eps}. For a
Gaussian prior-ratio factor, Spectra's transformed query is an identity of
this kind. Spectra applies it to the prior ratio of a frozen amortized
posterior model, extends it to linear tilts, factors that broaden the prior
and positive mixtures of such factors, and shows that, among affine query maps
with a scalar noise variance, isotropic exponential--quadratic factors are the
only ones that admit exact transport.

Diffusion scores can be composed across observations or factors, although
naive score addition may require corrected sampling
\citep{geffner2023compositional,linhart2026tall,du2023compositional}.
Particle-based methods provide another route to conditional diffusion
sampling \citep{wu2023tds,skreta2025fkc,cardoso2024mcgdiff}. In our mixture
construction, the adapted posterior is instead represented as a mixture
induced by the prior ratio, with component weights estimated once and one
component kept fixed along each reverse trajectory.

\section{Notation and assumptions}
\label{app:notation}

\begin{table}[ht]
\centering
\small
\caption{Principal notation. Symbols specific to error analysis and sampling protocols are defined where they are used.}
\begin{tabular}{ll}
\toprule
Symbol & Meaning \\
\midrule
$\pi_{\rm tr},\pi_{\rm new},r$ & training prior, target prior, and their ratio \\
$p,q$ & base and target posteriors \\
$q^\phi$ & posterior induced by a single prior-ratio factor \\
$t$, $\tau=\sigma(t)^2$ & diffusion time and corresponding noise variance \\
$p_\tau,s_p,\widehat s_p$ & Gaussian-smoothed base posterior, its exact score, and frozen score estimate \\
$\phi_k,b_k,q_k$ & prior-ratio factor, its coefficient, and posterior component \\
$Z_k,\widetilde Z_k$ & normalizing constants under the exact posterior and finite base sampler \\
$\alpha_k,\widehat\alpha_k$ & exact and estimated component weights \\
$D_\tau,m_\tau,\rho_\tau$ & transport denominator, transformed state, and transformed noise variance \\
\bottomrule
\end{tabular}
\end{table}

\paragraph{Density notation.}
In the main text, $p$ and $q$ denote exact Bayesian posteriors. In
Appendix~\ref{app:error-proof}, exact target and component posterior laws
carry a star when they are compared with implemented laws. 
Although $p$ is a Bayesian posterior in our application, the algebraic transport results only require it to be a base density satisfying the conditions below.

\paragraph{Support.}
We assume $\pi_{\rm new}\ll\pi_{\rm tr}$ and define $r=d\pi_{\rm new}/d\pi_{\rm tr}$. 
For priors with densities, this reduces to $r(\theta)=\pi_{\rm new}(\theta)/\pi_{\rm tr}(\theta)$ 
wherever $\pi_{\rm tr}(\theta)>0$. 
If the target posterior instead assigns mass outside $S=\supp\pi_{\rm tr}$, 
any output law satisfying $\widehat q(S)=1$ incurs $\TV(q,\widehat q)\ge q(S^c)$, 
by evaluating total variation on $S^c$.

\paragraph{Regularity.}
The base $p$ is a probability density on $\R^d$. 
The Gaussian convolutions used below are positive and differentiable, 
differentiation under the integral is valid, and each required normalizing constant is finite and nonzero. 
Kernel equalities hold almost everywhere with respect to Lebesgue measure.

\paragraph{Exactness.}
Exactness refers to the score identity obtained from the exact base score. 
Exact posterior sampling additionally requires the correct target terminal distribution 
and exact reverse dynamics reaching the clean endpoint. The implementation uses a learned score, 
a finite solver, and an approximate terminal distribution.

\section{Derivation of the noisy prior-reweighting identity}
\label{app:noisy-reweight-derivation}

Fix an observation $x$ and diffusion time $t$. Under the training prior, the noisy posterior is
\begin{equation}
p_t(z\mid x)=\int p_{t|0}(z\mid\theta)p(\theta\mid x)\,d\theta.
\label{eq:app-noisy-base}
\end{equation}
For a new prior $\pi_{\rm new}$, let
\begin{equation*}
r(\theta)=\frac{\pi_{\rm new}(\theta)}{\pi_{\rm tr}(\theta)},
\qquad
q(\theta\mid x)=\frac{p(\theta\mid x)r(\theta)}{Z_r(x)},
\end{equation*}
where
\begin{equation}
Z_r(x)=\int p(\theta\mid x)r(\theta)\,d\theta.
\label{eq:app-clean-normalizer}
\end{equation}
Applying the same forward diffusion to the new posterior gives
\begin{align}
q_t(z\mid x)
&=\int p_{t|0}(z\mid\theta)q(\theta\mid x)\,d\theta\nonumber\\
&=\frac{1}{Z_r(x)}\int p_{t|0}(z\mid\theta)p(\theta\mid x)r(\theta)\,d\theta.
\label{eq:app-noisy-target}
\end{align}
The reverse conditional under the base posterior follows from Bayes' rule,
\begin{equation}
p(\theta\mid z,x)=\frac{p_{t|0}(z\mid\theta)p(\theta\mid x)}{p_t(z\mid x)}.
\label{eq:app-reverse-conditional}
\end{equation}
Equivalently,
\begin{equation}
p_{t|0}(z\mid\theta)p(\theta\mid x)=p_t(z\mid x)p(\theta\mid z,x).
\label{eq:app-joint-factorization}
\end{equation}
Substituting Eq.~\eqref{eq:app-joint-factorization} into Eq.~\eqref{eq:app-noisy-target} yields
\begin{align}
q_t(z\mid x)
&=\frac{p_t(z\mid x)}{Z_r(x)}\int p(\theta\mid z,x)r(\theta)\,d\theta\nonumber\\
&=\frac{p_t(z\mid x)}{Z_r(x)}\E_{p(\theta\mid z,x)}[r(\theta)].
\label{eq:app-noisy-expectation}
\end{align}
Defining $H_t(z,x)=\E_{p(\theta\mid z,x)}[r(\theta)]$ therefore gives
\begin{equation}
q_t(z\mid x)=\frac{p_t(z\mid x)H_t(z,x)}{Z_r(x)}.
\label{eq:app-noisy-reweight}
\end{equation}

Since $Z_r(x)$ does not depend on $z$,
\begin{align}
s_q(z,t,x)
&=\nabla_z\log q_t(z\mid x)\nonumber\\
&=\nabla_z\log p_t(z\mid x)+\nabla_z\log H_t(z,x)\nonumber\\
&=s_p(z,t,x)+\nabla_z\log H_t(z,x),
\label{eq:app-score-decomposition}
\end{align}
which gives Eq.~\eqref{eq:prior-score-correction}.

\section{Exact single-factor transport}
\label{app:component-proof}

\subsection{Proof of Proposition~\ref{prop:transport}}

Fix an observation $x$ and a noise level $\tau>0$ with
$D_\tau=1+\kappa\tau>0$. Let $\phi$ have the form in
Eq.~\eqref{eq:transport-factor}, and define
\[
Z_\phi(x)
=
\int p(\theta\mid x)\phi(\theta)\,d\theta,
\]
assuming $0<Z_\phi(x)<\infty$. A constant scale in $r\propto\phi$
cancels in posterior normalization, so the Gaussian-smoothed target
posterior is
\begin{equation}
q_\tau^\phi(z\mid x)
=
\frac{1}{Z_\phi(x)}
\int p(\theta\mid x)\phi(\theta)
\mathcal N(z;\theta,\tau I)\,d\theta.
\label{eq:component-smoothing-app}
\end{equation}
We first rewrite the factor-kernel product inside the integral:
\begin{align}
\phi(\theta)\mathcal N(z;\theta,\tau I)
&=
(2\pi\tau)^{-d/2}
\exp\left(
c+a^\top\theta
-\frac{\kappa}{2}\lVert\theta\rVert^2
-\frac{\lVert z-\theta\rVert^2}{2\tau}
\right)
\nonumber\\
&=
(2\pi\tau)^{-d/2}
\exp\left(
c-\frac{\lVert z\rVert^2}{2\tau}
-\frac12\left(\tau^{-1}+\kappa\right)\lVert\theta\rVert^2
+\left(z/\tau+a\right)^\top\theta
\right).
\label{eq:factor-kernel-expanded-app}
\end{align}
To write the terms depending on $\theta$ as a Gaussian, set
\begin{equation}
\rho_\tau^{-1}
=
\tau^{-1}+\kappa,
\qquad
\rho_\tau^{-1}m_\tau(z)
=
z/\tau+a.
\end{equation}
Since $D_\tau=1+\kappa\tau$, this gives
\begin{equation}
\rho_\tau
=
\frac{\tau}{D_\tau},
\qquad
m_\tau(z)
=
\frac{z+\tau a}{D_\tau}.
\label{eq:transport-map-app}
\end{equation}
Completing the square in $\theta$ gives
\begin{align}
\phi(\theta)\mathcal N(z;\theta,\tau I)
&=
(2\pi\tau)^{-d/2}
\exp\left(
-\frac{\lVert\theta-m_\tau(z)\rVert^2}{2\rho_\tau}
+c
+\frac{\lVert z+\tau a\rVert^2}{2\tau D_\tau}
-\frac{\lVert z\rVert^2}{2\tau}
\right).
\label{eq:square-compensation-app}
\end{align}
Because
\[
(2\pi\tau)^{-d/2}
=
D_\tau^{-d/2}(2\pi\rho_\tau)^{-d/2},
\]
we obtain
\begin{align}
\phi(\theta)\mathcal N(z;\theta,\tau I)
&=
C_\tau(z)
\mathcal N(m_\tau(z);\theta,\rho_\tau I),
\label{eq:kernel-factor-app}\\
\log C_\tau(z)
&=
c-\frac d2\log D_\tau
+\frac{\lVert z+\tau a\rVert^2}{2\tau D_\tau}
-\frac{\lVert z\rVert^2}{2\tau}.
\label{eq:c-factor-app}
\end{align}
The Gaussian kernel is written with argument $m_\tau(z)$ and mean
$\theta$ to match the base smoothing. Since $C_\tau(z)$ is independent
of $\theta$, substituting Eq.~\eqref{eq:kernel-factor-app} into
Eq.~\eqref{eq:component-smoothing-app} gives
\begin{align}
q_\tau^\phi(z\mid x)
&=
\frac{C_\tau(z)}{Z_\phi(x)}
\int p(\theta\mid x)
\mathcal N(m_\tau(z);\theta,\rho_\tau I)\,d\theta
\nonumber\\
&=
\frac{C_\tau(z)}{Z_\phi(x)}
p_{\rho_\tau}(m_\tau(z)\mid x).
\label{eq:normalized-component-app}
\end{align}
At fixed $\tau$, $\rho_\tau$ is independent of $z$, while
\[
\nabla_z m_\tau(z)
=
\frac{1}{D_\tau}I.
\]
Moreover,
\begin{align}
\nabla_z\log C_\tau(z)
&=
\frac{z+\tau a}{\tau D_\tau}
-\frac{z}{\tau}
=
\frac{a-\kappa z}{D_\tau},
\label{eq:c-gradient-app}\\
\nabla_z\log p_{\rho_\tau}(m_\tau(z)\mid x)
&=
\frac{1}{D_\tau}
s_p(m_\tau(z),\rho_\tau,x).
\label{eq:base-gradient-app}
\end{align}
Taking the log gradient of
Eq.~\eqref{eq:normalized-component-app} and using
$\nabla_z\log Z_\phi(x)=0$ therefore yields
\begin{equation}
s_{q^\phi}(z,\tau,x)
=
\frac{a-\kappa z}{D_\tau}
+
\frac{1}{D_\tau}
s_p(m_\tau(z),\rho_\tau,x),
\end{equation}
which proves Proposition~\ref{prop:transport}. The derivation requires
no Gaussian assumption on $p$ and no derivative of its score.

\subsection{Learned-score error and admissible noise levels}

Define $e(z,\tau,x)=\widehat s_p(z,\tau,x)-s_p(z,\tau,x)$. Let
$\widehat s_{q^\phi}$ denote the transported score obtained by replacing
$s_p$ with $\widehat s_p$ in Proposition~\ref{prop:transport}.
Subtracting the exact transported score cancels the shared closed-form
term, giving
\begin{equation}
\widehat s_{q^\phi}(z,\tau,x)
-
s_{q^\phi}(z,\tau,x)
=
D_\tau^{-1}e(m_\tau(z),\rho_\tau,x).
\label{eq:score-error}
\end{equation}
For $\kappa\ge0$, the multiplier $D_\tau^{-1}$ is at most one. For
$\kappa<0$, it exceeds one on the admissible range. A multiplier below
one does not by itself imply smaller score error, because the base-score
error is evaluated at the transformed query $(m_\tau(z),\rho_\tau)$,
where its magnitude may differ.

The condition $D_\tau>0$ is required to keep the transformed variance
$\rho_\tau=\tau/D_\tau$ positive. For $\kappa\ge0$, it holds at every
positive noise level. For $\kappa<0$, every noise variance used by the
sampler must satisfy
$\tau<-1/\kappa$. Finiteness of $Z_\phi(x)$ is a separate requirement.
The transformed noise levels $\rho_\tau$ must also lie within the
noise range supported by the learned score model, since clipping them back
into that range changes the score identity.

\subsection{Special cases and interpretation}

For a linear exponential tilt,
$\phi(\theta)=\exp(c+a^\top\theta)$, we have
$\kappa=0$, $D_\tau=1$, $\rho_\tau=\tau$, and
$m_\tau(z)=z+\tau a$. Hence
\begin{equation}
s_{q^\phi}(z,\tau,x)
=
a+s_p(z+\tau a,\tau,x).
\label{eq:linear-tilt}
\end{equation}
The noise level is unchanged. The transport only shifts the
base-score query by $\tau a$ and adds the constant correction $a$.

For isotropic Gaussian localization,
$\phi(\theta)=\mathcal N(\theta;\mu_\phi,\lambda I)$ with $\lambda>0$,
the terms depending on $\theta$ give
$a=\mu_\phi/\lambda$ and $\kappa=1/\lambda$. The remaining terms are
absorbed into $c$. Substitution into Eq.~\eqref{eq:exact-score} gives
\begin{equation}
s_{q^\phi}(z,\tau,x)
=
\frac{\mu_\phi-z}{\tau+\lambda}
+
\frac{\lambda}{\tau+\lambda}
s_p\!\left(
\frac{\lambda z+\tau\mu_\phi}{\tau+\lambda},
\frac{\tau\lambda}{\tau+\lambda},
x
\right).
\label{eq:gaussian-localization}
\end{equation}
Here the transformed state lies between $z$ and the localization
center $\mu_\phi$, while the transformed noise variance
$\tau\lambda/(\tau+\lambda)$ is smaller than $\tau$.

\section{Formal characterization of exact transport}
\label{app:closure}

This section proves the characterization in
Theorem~\ref{thm:characterization}. The question is which positive
factors can be absorbed into an affine change of the Gaussian query
state and a scalar change of its variance. All statements concern this
explicit kernel identity.

\subsection{Formal statement}

Fix $\tau>0$ and let $\phi:\R^d\to(0,\infty)$ be measurable. Suppose
there exist a scalar $\rho>0$, a matrix $A\in\R^{d\times d}$, a vector
$u\in\R^d$, and a measurable function $C:\R^d\to(0,\infty)$ such that
\begin{equation}
\phi(\theta)\mathcal N(z;\theta,\tau I)
=
C(z)\mathcal N(Az+u;\theta,\rho I)
\label{eq:closure-kernel}
\end{equation}
for almost every $(z,\theta)$ with respect to Lebesgue measure. At the
fixed noise level $\tau$, the quantities $\rho$, $A$, and $u$ do not
depend on $z$ or $\theta$. Then there exist constants $c\in\R$, $a\in\R^d$, and $\kappa\in\R$
with $1+\kappa\tau>0$ such that
\begin{equation}
\log\phi(\theta)
=
c+a^\top\theta-\frac{\kappa}{2}\lVert\theta\rVert^2
\end{equation}
for almost every $\theta$. Moreover,
\begin{equation}
A
=
\frac{\rho}{\tau}I
=
\frac{1}{1+\kappa\tau}I,
\qquad
u=\rho a,
\qquad
\rho=\frac{\tau}{1+\kappa\tau}.
\end{equation}
Conversely, every such $\phi$ satisfies
Eq.~\eqref{eq:closure-kernel} with these parameters and the prefactor
in Eq.~\eqref{eq:c-factor-app}.

\subsection{Necessity}

\begin{proof}
Taking logarithms in Eq.~\eqref{eq:closure-kernel} and isolating
$\log\phi(\theta)$ gives
\begin{align}
\log\phi(\theta)
={}&
\log C(z)
+\frac d2\log\frac{\tau}{\rho}
-\frac{1}{2\rho}\lVert\theta-Az-u\rVert^2
+\frac{1}{2\tau}\lVert\theta-z\rVert^2.
\label{eq:closure-expand}
\end{align}
The term $d\log(\tau/\rho)/2$ comes from the two Gaussian
normalizers. Expanding the squared norms gives
\begin{align}
\log\phi(\theta)
={}&
-\frac12\left(\frac1\rho-\frac1\tau\right)
\lVert\theta\rVert^2
+\theta^\top
\left[
\left(\frac A\rho-\frac I\tau\right)z
+\frac u\rho
\right]
\nonumber\\
&+
\underbrace{
\log C(z)
+\frac d2\log\frac\tau\rho
-\frac{\lVert Az+u\rVert^2}{2\rho}
+\frac{\lVert z\rVert^2}{2\tau}
}_{F(z)}.
\label{eq:closure-coefficients-app}
\end{align}
The left side is independent of $z$. To make the corresponding
coefficient comparison rigorous under the almost-everywhere identity,
Fubini's theorem gives a full-measure set of $z$ values for which
Eq.~\eqref{eq:closure-coefficients-app} holds for almost every
$\theta$. For any two such values of $z$, subtracting the two
identities gives a polynomial in $\theta$ that vanishes almost
everywhere. Hence all of its coefficients are zero. The coefficients
of Eq.~\eqref{eq:closure-coefficients-app} must therefore be
independent of $z$ almost everywhere.

In particular, there is a constant vector $a\in\R^d$ such that
\begin{equation}
\left(\frac A\rho-\frac I\tau\right)z+\frac u\rho
=
a
\qquad
\text{for almost every }z.
\end{equation}
An affine function of $z$ can be constant almost everywhere only if
its linear part vanishes. Hence
\begin{equation}
\frac A\rho-\frac I\tau=0,
\qquad
A=\frac\rho\tau I,
\qquad
a=\frac u\rho.
\end{equation}
The quadratic coefficient is already independent of $z$ and has the
required isotropic form. Define
\begin{equation}
\kappa
=
\frac1\rho-\frac1\tau.
\end{equation}
Then the quadratic term is
$-\kappa\lVert\theta\rVert^2/2$. The remaining term $F(z)$ must
likewise equal a constant $c$ almost everywhere. Therefore
Eq.~\eqref{eq:closure-coefficients-app} reduces to
\begin{equation}
\log\phi(\theta)
=
c+a^\top\theta
-\frac\kappa2\lVert\theta\rVert^2
\end{equation}
almost everywhere, proving the necessary form of $\phi$. The definition of $\kappa$ also determines the noise transform:
\begin{equation}
1+\kappa\tau
=
\frac\tau\rho
>0,
\qquad
\rho
=
\frac\tau{1+\kappa\tau},
\qquad
Az+u
=
\frac{z+\tau a}{1+\kappa\tau}.
\end{equation}
Using $F(z)=c$ and solving for $C(z)$ gives
\begin{align}
\log C(z)
&=
c-\frac d2\log\frac\tau\rho
+\frac{\lVert Az+u\rVert^2}{2\rho}
-\frac{\lVert z\rVert^2}{2\tau}
\nonumber\\
&=
c-\frac d2\log(1+\kappa\tau)
+\frac{\lVert z+\tau a\rVert^2}
{2\tau(1+\kappa\tau)}
-\frac{\lVert z\rVert^2}{2\tau}.
\end{align}
This is exactly Eq.~\eqref{eq:c-factor-app}. The remaining dependence
on $z$ is independent of $\theta$ and is absorbed into the prefactor
$C(z)$.
\end{proof}

\subsection{Sufficiency and scope}

Conversely, take
$\phi(\theta)=\exp(c+a^\top\theta-\kappa\lVert\theta\rVert^2/2)$
with $D_\tau=1+\kappa\tau>0$. Set
$\rho=\tau/D_\tau$, $A=I/D_\tau$, $u=\rho a$, and $C=C_\tau$.
Then
\[
Az+u
=
\frac{z+\tau a}{D_\tau}
=
m_\tau(z),
\]
so Eq.~\eqref{eq:kernel-factor-app} is exactly
Eq.~\eqref{eq:closure-kernel}. This proves sufficiency and completes
Theorem~\ref{thm:characterization}.

The necessity result already holds at a single fixed $\tau>0$: the
kernel identity need not hold over an interval of noise levels to force
the exponential-quadratic form. Across a diffusion schedule, the factor
$\phi$ is fixed, so its coefficients $(a,\kappa,c)$ do not depend on
$\tau$. Only the induced kernel transform changes with the noise level:
\[
A_\tau=\frac{1}{1+\kappa\tau}I,
\qquad
u_\tau=\frac{\tau a}{1+\kappa\tau},
\qquad
\rho_\tau=\frac{\tau}{1+\kappa\tau}.
\]
The identity therefore applies along a schedule provided
$D_\tau=1+\kappa\tau>0$ at every queried noise level. When $\kappa=0$,
these reduce to $\rho_\tau=\tau$, $A_\tau=I$, and $u_\tau=\tau a$.

Because the identity holds at the kernel level, integrating it against any
base posterior that satisfies the regularity conditions of
Appendix~\ref{app:notation} gives the same score transport. The characterization
concerns affine state maps with scalar isotropic variance as specified
in Eq.~\eqref{eq:closure-kernel}. It makes no statement about algorithms
that do not admit this kernel equality.

\section{Static mixtures and dynamic responsibilities}
\label{app:mixture-proof}

Fix $x$ and a finite positive representation
$r=\sum_{k=1}^K b_k\phi_k$ with $b_k>0$ and
$0<Z_k(x)<\infty$. The following calculations distinguish the
observation-specific component weights $\alpha_k(x)$ from the
noisy-state-dependent component responsibilities
$\gamma_{k,\tau}(z,x)$.

\subsection{Normalizing the positive sum}

Recall that
$q_k(\theta\mid x)=p(\theta\mid x)\phi_k(\theta)/Z_k(x)$.
Each component is normalized because
\begin{equation}
\int q_k(\theta\mid x)\,d\theta
=
\frac{1}{Z_k(x)}
\int p(\theta\mid x)\phi_k(\theta)\,d\theta
=
1.
\end{equation}
Substitution into the unnormalized target gives
\begin{align}
p(\theta\mid x)r(\theta)
&=
p(\theta\mid x)\sum_k b_k\phi_k(\theta)
\nonumber\\
&=
\sum_k b_kZ_k(x)q_k(\theta\mid x).
\end{align}
Integrating both sides with respect to $\theta$ determines the total
normalizer:
\begin{equation}
Z_r(x)
=
\int p(\theta\mid x)r(\theta)\,d\theta
=
\sum_k b_kZ_k(x)
\int q_k(\theta\mid x)\,d\theta
=
\sum_k b_kZ_k(x).
\end{equation}
Dividing by $Z_r(x)$ gives
\begin{equation}
q(\theta\mid x)
=
\sum_k\alpha_k(x)q_k(\theta\mid x),
\qquad
\alpha_k(x)
=
\frac{b_kZ_k(x)}{Z_r(x)}.
\end{equation}
The weights are positive and sum to one. Their dependence on $x$
comes from the expectations $Z_k(x)$ under the base posterior. This
proves Eq.~\eqref{eq:posterior-mixture}.

\subsection{Why the same weights persist under Gaussian smoothing}

The noisy component density is
$q_{k,\tau}(z\mid x)
=
\int\mathcal N(z;\theta,\tau I)q_k(\theta\mid x)\,d\theta$.
Since the mixture is finite, we can interchange its sum and the
smoothing integral:
\begin{align}
q_\tau(z\mid x)
&=
\int
\mathcal N(z;\theta,\tau I)
\sum_k\alpha_k(x)q_k(\theta\mid x)\,d\theta
\nonumber\\
&=
\sum_k\alpha_k(x)
\int
\mathcal N(z;\theta,\tau I)
q_k(\theta\mid x)\,d\theta
\nonumber\\
&=
\sum_k\alpha_k(x)q_{k,\tau}(z\mid x).
\end{align}
Each $q_{k,\tau}$ is normalized, so this is a mixture with the same
weights, not an unnormalized sum requiring new weights at each noise
level.

Introduce a latent label $J$ with
$\Pr(J=k\mid x)=\alpha_k(x)$ and draw
$\theta\mid J=k,x$ from $q_k$. After applying Gaussian corruption to
$\theta$, the joint density is
\begin{equation}
\Pr(J=k,z_\tau\in dz\mid x)
=
\alpha_k(x)q_{k,\tau}(z\mid x)\,dz.
\label{eq:label-noisy-joint-app}
\end{equation}
Marginalizing $z$ leaves
$\Pr(J=k\mid x)=\alpha_k(x)$ at every noise level. Conditioning on a
particular $z$, however, gives
\begin{equation}
\Pr(J=k\mid z_\tau=z,x)
=
\frac{\alpha_k(x)q_{k,\tau}(z\mid x)}
{q_\tau(z\mid x)}
=
\gamma_{k,\tau}(z,x).
\label{eq:label-responsibility-app}
\end{equation}
Here $\alpha_k(x)$ are the static component weights, whereas
$\gamma_{k,\tau}(z,x)$ are the noisy-state-dependent component
responsibilities.

\subsection{Correctness of sampling the component once}

For any measurable set $B\subseteq\R^d$, the two-stage construction in
Eq.~\eqref{eq:static-component-sampling} satisfies
\begin{align}
\Pr(\theta\in B\mid x)
&=
\sum_k
\Pr(J=k\mid x)
\Pr(\theta\in B\mid J=k,x)
\nonumber\\
&=
\sum_k
\alpha_k(x)
\int_B q_k(\theta\mid x)\,d\theta
=
\int_B q(\theta\mid x)\,d\theta.
\end{align}
Hence sampling one label followed by one draw from its component
produces the target mixture.

Reverse diffusion supplies the conditional component draw. At a
terminal noise level $\tau_T=\sigma(t_{\max})^2$ (Appendix~\ref{app:terminal-init}), the ideal construction draws $J$ with
probabilities $\alpha(x)$ and then initializes $z_T$ from
$q_{J,\tau_T}(\cdot\mid x)$. Conditional on $J=k$, exact reverse
dynamics driven by $s_{q_k}$ recover $q_{k,\tau}$ at each noise level
and $q_k$ at the clean endpoint. Since the component label is drawn
with probability $\alpha_k(x)$ and then kept fixed, marginalizing over
$J$ recovers $q_\tau=\sum_k\alpha_k(x)q_{k,\tau}$ throughout the
reverse process, and hence $q$ at the endpoint. Proposition~\ref{prop:transport} supplies $s_{q_k}$ from one
transformed base-score query.

The exact mixture algebra also identifies what the implementation
produces. Let $\widehat q_k$ be the endpoint law of the finite component
sampler, with fresh sampling randomness independent of the weight
estimates. If labels are drawn with estimated weights
$\widehat\alpha$, then, conditional on those estimates, the output law
is exactly
\begin{equation}
\widehat q
=
\sum_k\widehat\alpha_k\widehat q_k.
\label{eq:implemented-mixture-app}
\end{equation}
The approximation to the Bayesian target enters through the weights and
component endpoint laws, not through the act of drawing a static label.
Appendix~\ref{app:error-proof} quantifies their contributions
separately.

\subsection{Deriving the marginal mixture score}

By the definition of the score,
$\nabla_z q_{k,\tau}=q_{k,\tau}s_{q_k}$. Differentiating the noisy
density mixture therefore gives
\begin{align}
s_q(z,\tau,x)
&=
\frac{\nabla_z q_\tau(z\mid x)}
{q_\tau(z\mid x)}
\nonumber\\
&=
\frac{
\sum_k\alpha_k(x)\nabla_z q_{k,\tau}(z\mid x)
}{
q_\tau(z\mid x)
}
\nonumber\\
&=
\sum_k
\frac{
\alpha_k(x)q_{k,\tau}(z\mid x)
}{
q_\tau(z\mid x)
}
s_{q_k}(z,\tau,x)
\nonumber\\
&=
\sum_k
\gamma_{k,\tau}(z,x)s_{q_k}(z,\tau,x).
\label{eq:dynamic-resp}
\end{align}
Together with Eq.~\eqref{eq:label-responsibility-app}, this gives the
marginal mixture score. It also shows why replacing
$\gamma_{k,\tau}(z,x)$ by the static weights $\alpha_k(x)$ would
generally give a different score.

To expose the density information required by $\gamma$, apply
Eq.~\eqref{eq:normalized-component-app} to component $k$:
\begin{equation}
q_{k,\tau}(z\mid x)
=
\frac{C_{k,\tau}(z)}{Z_k(x)}
p_{\rho_{k,\tau}}
\bigl(m_{k,\tau}(z)\mid x\bigr).
\end{equation}
Multiplying by
$\alpha_k(x)=b_kZ_k(x)/Z_r(x)$ cancels $Z_k(x)$:
\begin{equation}
\alpha_k(x)q_{k,\tau}(z\mid x)
=
\frac{1}{Z_r(x)}
\underbrace{
b_kC_{k,\tau}(z)
p_{\rho_{k,\tau}}
\bigl(m_{k,\tau}(z)\mid x\bigr)
}_{U_{k,\tau}(z,x)}.
\end{equation}
The common $Z_r(x)$ also cancels when normalizing over components, so
\begin{equation}
\gamma_{k,\tau}(z,x)
=
\frac{
U_{k,\tau}(z,x)
}{
\sum_j U_{j,\tau}(z,x)
}.
\end{equation}
Although the normalizing constants $Z_k(x)$ cancel, evaluating
$\gamma_{k,\tau}$ still requires normalized base-density values at
different transformed states and, in general, different noise levels.
For two components,
\begin{align}
\log\frac{U_{k,\tau}(z,x)}{U_{j,\tau}(z,x)}
={}&
\log\frac{
b_kC_{k,\tau}(z)
}{
b_jC_{j,\tau}(z)
}
\nonumber\\
&+
\log p_{\rho_{k,\tau}}
\bigl(m_{k,\tau}(z)\mid x\bigr)
-
\log p_{\rho_{j,\tau}}
\bigl(m_{j,\tau}(z)\mid x\bigr).
\end{align}

The first term is available analytically from the factor parameters.
The remaining terms require log-density values at the transformed
queries. A score evaluation provides only the local gradient
$\nabla\log p_\rho$, not the corresponding log-density value.
Evaluating the dynamic responsibilities would therefore require
additional machinery to recover or estimate the relevant relative
densities. Static component sampling avoids this step: it requires
only the observation-level weights $\alpha_k(x)$ and keeps the selected
component fixed throughout the reverse trajectory.

\subsection{Factor scaling and reuse across prior coefficients}

A positive rescaling $\phi_k'=\zeta_k\phi_k$, accompanied by
$b_k'=b_k/\zeta_k$, leaves the represented ratio unchanged. It gives
$Z_k'=\zeta_kZ_k$ and
\begin{equation}
q_k'
=
\frac{p\,\zeta_k\phi_k}{\zeta_kZ_k}
=
q_k,
\qquad
b_k'Z_k'
=
b_kZ_k.
\end{equation}
Neither the component distribution nor its component weight
depends on the factor-scaling convention. Any fixed factor scale
may be absorbed into $b_k$.

For a fixed observation and unchanged factors,
$Z_k(x)=\E_{p(\theta\mid x)}[\phi_k(\theta)]$ does not depend on
$b_k$. If only the coefficient vector changes, the new component weights
are
\begin{equation}
\alpha_k'(x)
=
\frac{b_k'Z_k(x)}
{\sum_j b_j'Z_j(x)},
\end{equation}
while the component score transforms remain unchanged. The same reuse
holds for estimated normalizers when the observation, component set,
and weight-estimation protocol are held fixed. A new observation
generally requires new normalizers. If the supplied representation
approximates $r$ by $\widetilde r$, all mixture identities in this
section apply to the posterior induced by $\widetilde r$. Its
discrepancy from the intended target is handled in
Appendix~\ref{app:error-proof}.

\subsection{Constructing factors and coefficients from a known prior mixture}
\label{app:factor-coefficients}

The coefficients $b_k$ belong to the supplied prior-ratio
representation. They are distinct from the observation-dependent
component weights $\alpha_k(x)$. Suppose the target prior has the known
mixture representation
\begin{equation}
\pi_{\rm new}(\theta)
=
\sum_k \omega_k \nu_k(\theta),
\qquad
\omega_k>0,
\qquad
\sum_k \omega_k=1.
\end{equation}
On the support of the training prior, division by $\pi_{\rm tr}$ gives
\begin{equation}
r(\theta)
=
\sum_k
\omega_k
\frac{\nu_k(\theta)}{\pi_{\rm tr}(\theta)}.
\end{equation}
If each ratio $\nu_k/\pi_{\rm tr}$ has the form in
Eq.~\eqref{eq:transport-factor}, we may take
\begin{equation}
\phi_k(\theta)
=
\frac{\nu_k(\theta)}{\pi_{\rm tr}(\theta)},
\qquad
b_k=\omega_k.
\end{equation}
This is an analytic construction, requiring no posterior samples or
observation-specific estimation. A mixture of arbitrary prior densities
need not satisfy this factor condition.

For two Gaussian priors,
$\pi_{\rm tr}=\mathcal N(\mu_0,\Sigma_0)$ and
$\nu=\mathcal N(\mu_1,\Sigma_1)$, their log ratio has the form
\begin{align}
\log\frac{\nu(\theta)}{\pi_{\rm tr}(\theta)}
={}&
c
+
\left(
\Sigma_1^{-1}\mu_1
-
\Sigma_0^{-1}\mu_0
\right)^\top\theta
\nonumber\\
&-
\frac12
\theta^\top
\left(
\Sigma_1^{-1}
-
\Sigma_0^{-1}
\right)
\theta,
\label{eq:general-gaussian-ratio}
\end{align}
where $c$ is independent of $\theta$. Hence this ratio has the
single-factor form in Eq.~\eqref{eq:transport-factor} exactly when
\begin{equation}
\Sigma_1^{-1}-\Sigma_0^{-1}
=
\kappa I.
\end{equation}
The means may therefore change arbitrarily. The restriction is that the
change in quadratic curvature is isotropic. The isotropic Gaussian case
below is a common special case.

Let
$\pi_{\rm tr}=\mathcal N(\mu_0,v_0I)$ and
$\nu_k=\mathcal N(\mu_k,v_kI)$, with positive scalar variances.
Expanding their log ratio gives
\begin{align}
\log\frac{\nu_k(\theta)}{\pi_{\rm tr}(\theta)}
&=
\frac d2\log\frac{v_0}{v_k}
-
\frac{\lVert\theta-\mu_k\rVert^2}{2v_k}
+
\frac{\lVert\theta-\mu_0\rVert^2}{2v_0}
\nonumber\\
&=
c_k
+
a_k^\top\theta
-
\frac{\kappa_k}{2}\lVert\theta\rVert^2,
\end{align}
where
\begin{align}
\kappa_k
&=
v_k^{-1}-v_0^{-1},
\qquad
a_k
=
\mu_k/v_k-\mu_0/v_0,
\\
c_k
&=
\frac d2\log\frac{v_0}{v_k}
-
\frac{\lVert\mu_k\rVert^2}{2v_k}
+
\frac{\lVert\mu_0\rVert^2}{2v_0}.
\end{align}
If the constant $c_k$ is retained in $\phi_k$, then
$b_k=\omega_k$. The implementation instead uses a normalized Gaussian
factor when $\kappa_k>0$. Completing the square gives
\begin{equation}
\frac{\nu_k(\theta)}{\pi_{\rm tr}(\theta)}
=
\Gamma_k
\mathcal N\!\left(
\theta;
\frac{a_k}{\kappa_k},
\kappa_k^{-1}I
\right),
\end{equation}
with
\begin{equation}
\Gamma_k
=
\left(
\frac{2\pi}{\kappa_k}
\right)^{d/2}
\exp\!\left(
c_k
+
\frac{\lVert a_k\rVert^2}{2\kappa_k}
\right).
\end{equation}
Under this convention, $\phi_k$ is the normalized Gaussian and
$b_k=\omega_k\Gamma_k$. In general the coefficient differs from the prior
mixture weight $\omega_k$, but it is still determined analytically from
the supplied priors. For $\kappa_k\le0$, the exponential-quadratic form
remains the relevant representation rather than a normalized Gaussian
factor, subject to posterior normalizability and the sampling-schedule
condition.

For a uniform training prior on a box $B$ of volume $|B|$,
$\pi_{\rm tr}(\theta)=|B|^{-1}$ inside $B$. Hence
\[
r(\theta)
=
|B|\,\pi_{\rm new}(\theta),
\qquad \theta\in B.
\]
If the target prior is a mixture of isotropic Gaussians restricted to $B$ and
normalized as a whole, both $|B|$ and the truncation normalizer are
common to all components. They therefore cancel in posterior
normalization, so the Gaussian kernels can be used directly as
$\phi_k$ with their supplied mixture coefficients as $b_k$, up to a
common scale. This representation holds only on $B$ and does not
extend the training support. Appendix~\ref{app:component-approx}
gives the corresponding construction when the Gaussian mixture itself
is obtained by approximation.

\section{End-to-end error decomposition}
\label{app:error-proof}

\subsection{Mixture error bound}

Fix an observation and let $r$ be the intended prior ratio and
$\widetilde r$ its positive-factor approximation. The exact target is
$q^{r,\star}$. The positive-factor approximation induces
\[
q^{\widetilde r,\star}
=
\sum_k \alpha_k^\star q_k^\star,
\qquad
\alpha_k^\star(x)
=
\frac{b_kZ_k(x)}{\sum_j b_jZ_j(x)},
\]
where $q_k^\star$ denotes the exact posterior component induced by
$\phi_k$. For a fixed weight-estimation protocol, let $\Pp_0^{(R)}$ be its base endpoint law after $R$ transitions (Appendix~\ref{app:path}). Assume its factor expectations are finite and positive, and define
\begin{equation}
\widetilde Z_k(x)=\E_{\Pp_0^{(R)}}[\phi_k(\theta)],
\qquad
\alpha_k^P(x)=\frac{b_k\widetilde Z_k(x)}{\sum_j b_j\widetilde Z_j(x)}.
\label{eq:model-normalizers}
\end{equation}
Let $\widehat\alpha$ be the estimated weights and $\widehat q_k$ the
endpoint law of the final component sampler. The final component
sampler may use a different discretization and corrector configuration
from the estimator's proposal paths.
Conditional on the weight estimates, the output is $\widehat q=\sum_k\widehat\alpha_k\widehat q_k$, and
\begin{align}
\TV(q^{r,\star},\widehat q)
\leq{}&\underbrace{\TV(q^{r,\star},q^{\widetilde r,\star})}_{\text{prior-ratio approximation}}
+\underbrace{\tfrac12\lVert\widehat\alpha-\alpha^\star\rVert_1}_{\text{component weights}}
\nonumber\\[-1mm]
&+\underbrace{\sum_k\widehat\alpha_k\TV(q_k^\star,\widehat q_k)}_{\text{component sampling}},
\label{eq:e2e-bound}
\end{align}
where
\begin{equation}
\lVert\widehat\alpha-\alpha^\star\rVert_1
\leq
\underbrace{\lVert\widehat\alpha-\alpha^P\rVert_1}_{\text{weight estimator}}
+\underbrace{\lVert\alpha^P-\alpha^\star\rVert_1}_{\text{base-sampler discrepancy}}.
\label{eq:weight-split}
\end{equation}
The effective sample size (ESS) is an overlap diagnostic for the estimator, not a bound on either weight-error term. Assessing the base-sampler discrepancy requires an exact or independently validated posterior reference.

Insert the target induced by the approximated prior ratio and apply the triangle inequality:
\begin{equation}
\TV(q^{r,\star},\widehat q)
\le \TV(q^{r,\star},q^{\widetilde r,\star})
+\TV(q^{\widetilde r,\star},\widehat q).
\end{equation}
For the second term, insert the intermediate mixture
$\bar q=\sum_k\widehat\alpha_kq_k^\star$:
\begin{align}
\TV(q^{\widetilde r,\star},\widehat q)
&\le \TV\!\left(\sum_k\alpha_k^\star q_k^\star,
\sum_k\widehat\alpha_kq_k^\star\right)
+\TV\!\left(\sum_k\widehat\alpha_kq_k^\star,
\sum_k\widehat\alpha_k\widehat q_k\right)\\
&\le \frac12\lVert\widehat\alpha-\alpha^\star\rVert_1
+\sum_k\widehat\alpha_k\TV(q_k^\star,\widehat q_k).
\end{align}
Equation~\eqref{eq:weight-split} is the triangle inequality on the simplex.

\subsection{Binary component-weight sensitivity}

For $q_\alpha=\alpha q_1+(1-\alpha)q_2$,
\begin{equation}
q_{\widehat\alpha}-q_\alpha=(\widehat\alpha-\alpha)(q_1-q_2).
\end{equation}
Taking half the $L^1$ norm gives
\begin{equation}
\TV(q_{\widehat\alpha},q_\alpha)
=|\widehat\alpha-\alpha|\,\TV(q_1,q_2).
\label{eq:binary-tv}
\end{equation}
A given absolute weight error therefore matters most when the component
distributions are well separated and has little effect when they strongly
overlap.
If
\[
\delta
=
\log\frac{\widehat\alpha}{1-\widehat\alpha}
-
\log\frac{\alpha}{1-\alpha}
\]
is the log-odds error, then the logistic map has derivative at most
$1/4$, so
\begin{equation}
|\widehat\alpha-\alpha|
\le
|\delta|/4.
\end{equation}
For small $\delta$, the first-order response is
$|\widehat\alpha-\alpha|\approx\alpha(1-\alpha)|\delta|$, so for a small
log-odds error the weight itself is most sensitive near $\alpha=1/2$.

\subsection{Approximate-ratio posterior TV}

Let $q^{r,\star}=pr/Z_r$ and $q^{\widetilde r,\star}=p\widetilde r/Z_{\widetilde r}$ for a fixed observation, with $0<Z_r,Z_{\widetilde r}<\infty$, and define $\varepsilon=\E_p|r-\widetilde r|<\infty$. Since $|Z_r-Z_{\widetilde r}|\le\varepsilon$,
\begin{align}
\lVert q^{r,\star}-q^{\widetilde r,\star}\rVert_1
&\le \frac{\E_p|r-\widetilde r|}{Z_r}
+Z_{\widetilde r}\left|\frac1{Z_r}-\frac1{Z_{\widetilde r}}\right|\\
&\le \frac{2\varepsilon}{Z_r}.
\end{align}
Therefore
$\TV(q^{r,\star},q^{\widetilde r,\star})\le\varepsilon/Z_r$. The expectation defining
$\varepsilon$ must be under the observation-specific base posterior.
An approximation error measured only under the training prior does not
by itself control every observation-specific target posterior.
\section{Normalizer-estimation details}
\label{app:path}

The path-space estimator estimates the finite-chain normalizers
$\widetilde Z_k(x)$ by a change of measure on reverse trajectories.
We first derive the path-space identity and its overlap diagnostic,
then describe the alternative direct estimator, and finally summarize
the complete inference procedure. Sampling budgets and aggregation are
specified in Appendix~\ref{app:sampling-settings}.

\subsection{Finite-chain path-space identity}

Let $R$ denote the number of transitions in the normalizer estimator,
independently of the output sampler's $(N,L)$ configuration. Write the
estimator time grid as
$t_R>\cdots>t_0$, with $\tau_n=\sigma(t_n)^2$, and a reverse trajectory
as
\[
\omega=(z_R,z_{R-1},\ldots,z_0).
\]
Suppressing the fixed observation, the estimator's base and component
proposal path laws are
\begin{align}
\Pp^{(R)}(d\omega)
&=
p_{\rm term}(z_R)
\prod_{n=0}^{R-1}
P_n(z_n\mid z_{n+1})\,d\omega,
\\
\Qp_k^{(R)}(d\omega)
&=
q_{k,\rm term}(z_R)
\prod_{n=0}^{R-1}
Q_{k,n}(z_n\mid z_{n+1})\,d\omega.
\end{align}
Here $p_{\rm term}$ and $q_{k,\rm term}$ denote the terminal
distributions used by the two chains. When the terminal distributions
and Markov transitions have densities with common support, the
trajectory likelihood ratio is
\begin{equation}
\log\frac{d\Pp^{(R)}}{d\Qp_k^{(R)}}
=
\log\frac{p_{\rm term}(z_R)}{q_{k,\rm term}(z_R)}
+
\sum_{n=0}^{R-1}
\log
\frac{P_n(z_n\mid z_{n+1})}
{Q_{k,n}(z_n\mid z_{n+1})}.
\label{eq:path-rn}
\end{equation}
If the two chains use the same terminal distribution, the first term
is identically zero. Equation~\eqref{eq:path-rn} is the formal version
of the trajectory likelihood ratio $P/Q_k$ used in
Section~\ref{sec:weights}.

Change of measure then gives
\begin{equation}
\widetilde Z_k(x)
=
\int
\phi_k(z_0)\Pp^{(R)}(d\omega)
=
\E_{\Qp_k^{(R)}}\!\left[
\phi_k(z_0)
\frac{d\Pp^{(R)}}{d\Qp_k^{(R)}}
\right].
\label{eq:path-z-implemented}
\end{equation}
Trajectories can be drawn from the component proposal
$\Qp_k^{(R)}$, and the likelihood ratio corrects the expectation back
to the base path law $\Pp^{(R)}$. This is Eq.~\eqref{eq:path-z} with
the trajectory density ratio written in measure notation.

\paragraph{Implemented transitions.}
Both path laws use Gaussian Euler--Maruyama transitions with the same
covariance and differ only in their score-driven means. For the transition
from $t_{n+1}$ to $t_n$, define
\begin{equation}
\Delta_n
=
(t_{n+1}-t_n)\,g(t_{n+1})^2,
\qquad
g(t)^2
=
2\sigma(t)\dot\sigma(t).
\end{equation}
An Euler--Maruyama step of
Eq.~\eqref{eq:background-reverse-sde} for the base chain is
\begin{equation}
z_n
=
z_{n+1}
+
\Delta_n\,
\widehat s_p(z_{n+1},\tau_{n+1},x)
+
\sqrt{\Delta_n}\,\epsilon_n,
\qquad
\epsilon_n\sim\mathcal N(0,I).
\label{eq:path-transition}
\end{equation}
Hence $P_n$ is Gaussian with mean
\[
z_{n+1}
+
\Delta_n\widehat s_p(z_{n+1},\tau_{n+1},x)
\]
and covariance $\Delta_nI$.

Let $\widehat s_{q_k}$ denote the transported component-score estimate
obtained from Eq.~\eqref{eq:exact-score} by replacing $s_p$ with
$\widehat s_p$. The component proposal $Q_{k,n}$ uses the same time
grid and covariance, with mean
\[
z_{n+1}
+
\Delta_n\widehat s_{q_k}(z_{n+1},\tau_{n+1},x).
\]
For a transition generated under $Q_{k,n}$,
\begin{equation}
z_n
=
z_{n+1}
+
\Delta_n\widehat s_{q_k}(z_{n+1},\tau_{n+1},x)
+
\sqrt{\Delta_n}\,\epsilon_n,
\end{equation}
the corresponding log-density ratio is
\begin{equation}
\log
\frac{P_n(z_n\mid z_{n+1})}
{Q_{k,n}(z_n\mid z_{n+1})}
=
\sqrt{\Delta_n}\,
\epsilon_n^\top
\left(
\widehat s_p-\widehat s_{q_k}
\right)
-
\frac{\Delta_n}{2}
\left\Vert
\widehat s_p-\widehat s_{q_k}
\right\Vert^2,
\label{eq:path-transition-ratio}
\end{equation}
where both scores are evaluated at
$(z_{n+1},\tau_{n+1},x)$. The normalizer estimator is defined on this
corrector-free chain. Langevin corrector updates are used only by the final
output sampler.

\paragraph{Score evaluations per transition.}
Computing the proposal transition requires
$\widehat s_{q_k}(z_{n+1},\tau_{n+1},x)$, which uses one frozen-model
query at the transformed state and noise level
\[
\left(
m_{k,\tau_{n+1}}(z_{n+1}),
\rho_{k,\tau_{n+1}}
\right).
\]
Evaluating the base transition density in
Eq.~\eqref{eq:path-transition-ratio} additionally requires
$\widehat s_p(z_{n+1},\tau_{n+1},x)$. The implemented path-space
estimator therefore uses two frozen-model score evaluations per
transition. This additional evaluation applies only to weight
estimation. Final posterior sampling uses one frozen-model query per
transported-score evaluation, as reported in
Appendix~\ref{app:compute-details}.

To characterize the overlap relevant to this importance sampler,
define the endpoint-tilted path law
\begin{equation}
\Pp_k^\phi(d\omega)
=
\frac{\phi_k(z_0)}
{\widetilde Z_k(x)}
\Pp^{(R)}(d\omega).
\end{equation}
The normalized importance weight is
\[
\bar w_k
=
\frac{d\Pp_k^\phi}{d\Qp_k^{(R)}}.
\]
For probability laws $P\ll Q$, define the chi-squared divergence as
\begin{equation}
\chi^2(P\Vert Q)
=
\E_Q\!\left[
\left(
\frac{dP}{dQ}-1
\right)^2
\right].
\end{equation}
Since
$\E_{\Qp_k^{(R)}}[\bar w_k]=1$,
\begin{equation}
\E_{\Qp_k^{(R)}}[\bar w_k^2]
=
1+
\chi^2
\left(
\Pp_k^\phi
\Vert
\Qp_k^{(R)}
\right).
\end{equation}
For independent proposal trajectories $\omega_1,\ldots,\omega_{n_{\rm P}}\sim\Qp_k^{(R)}$, the
empirical effective sample size is
$\ESS_k=(\sum_i\bar w_k(\omega_i))^2/\sum_i\bar w_k(\omega_i)^2$. With a finite second
moment, $\ESS_k/n_{\rm P}$ therefore converges to
the reciprocal of this quantity
\citep{agapiou2017importance}. A larger ESS indicates that the
proposal $\Qp_k^{(R)}$ has better overlap with the endpoint-tilted path
law and that the importance weights are less concentrated. This
diagnostic concerns the finite-chain importance-sampling problem. It
does not measure the discrepancy between the base endpoint law
$\Pp_0^{(R)}$ and the exact Bayesian posterior $p(\theta\mid x)$.
The proposal path law used for normalizer estimation may also differ
from the path law used for final posterior sampling.

\subsection{Direct estimation and overlap}
\label{app:direct-estimation}

Given independent samples $\theta_n$ from the fixed base-sampling
protocol, \direct{} estimates the corresponding factor expectation by
\begin{equation}
\widehat Z_k^{\rm dir}
=
\frac{1}{n_{\rm B}}
\sum_{n=1}^{n_{\rm B}}
\phi_k(\theta_n).
\label{eq:direct}
\end{equation}
Here $n_{\rm B}$ is the bank size. Once the bank has been sampled, it
can be reused for all $K$ factors at $O(n_{\rm B}K)$ factor-evaluation
cost. The estimate can be unstable when the bank poorly covers regions
where $\phi_k$ is large.

For the base samples in Eq.~\eqref{eq:direct}, write
$\phi_{kn}=\phi_k(\theta_n)$. The usual effective-sample-size diagnostic
\citep{agapiou2017importance} is
\begin{equation}
\ESS_k^{\rm dir}
=
\frac{\left(\sum_n \phi_{kn}\right)^2}
{\sum_n \phi_{kn}^2}.
\end{equation}
It equals $n_{\rm B}$ when all samples contribute equally and
approaches one when a single sample dominates. A low ESS therefore
indicates that the estimate is dominated by only a small part of the
bank. A high ESS means that the observed contributions are spread more
evenly across the sampled points, which is favorable for stability.
However, ESS only describes the samples that are actually present in
the bank. It cannot detect an important region that the entire bank
failed to sample, and it does not measure the discrepancy between the
base-sampling endpoint distribution and the exact Bayesian posterior.

\subsection{Complete Spectra inference procedure}

\begin{paperalgorithm}[Spectra inference for one observation]
\label{alg:spectra}
\textbf{Input:} frozen score $\widehat s_p$, observation $x$,
supplied positive representation
$\{b_k,\phi_k\}_{k=1}^K$ of $r$ or $\widetilde r$, output count $M$,
and estimator and sampler settings.
\par\smallskip

\begin{tabular}{@{}r@{\quad}p{.93\linewidth}@{}}
1 &
Check $b_k>0$ and the factor form in
Eq.~\eqref{eq:transport-factor}. At every noise level used by the output
sampler, and by the normalizer estimator when $K>1$, check
$D_{k,\tau}>0$ and that $\rho_{k,\tau}$ lies in the noise range supported
by the learned score model.
\\

2 &
\textbf{if} $K=1$: set $\widehat\alpha_1=1$.
\\

3 &
\textbf{else}: estimate $\widehat Z_k$ using the chosen normalizer
estimator and set
$\widehat\alpha_k=
\dfrac{b_k\widehat Z_k}
{\sum_j b_j\widehat Z_j}$.
\\[1mm]

4 &
\textbf{for} $m=1,\ldots,M$:
\\

5 &
\quad Draw $J$ once with
$\Pr(J=k)=\widehat\alpha_k$.
\\

6 &
\quad Draw $z_T$ from the shared base terminal Gaussian approximation
in Eq.~\eqref{eq:terminal-init}.
\\

7 &
\quad \textbf{for each} reverse step:
\\

8 &
\qquad Advance the configured reverse sampler using component $J$'s
transported score estimate, obtained from
Eq.~\eqref{eq:exact-score} with $s_p$ replaced by $\widehat s_p$
whenever a score evaluation is required.
\\

9 &
\quad Store the endpoint as $\theta^{(m)}$.
\\

10 &
\textbf{return} $\{\theta^{(m)}\}_{m=1}^{M}$.
\\
\end{tabular}
\end{paperalgorithm}
\section{Approximating correlated Gaussians with exponential--quadratic factors}
\label{app:component-approx}

A correlated Gaussian prior change generally does not satisfy the
single-factor condition in Appendix~\ref{app:factor-coefficients}.
This section gives one concrete way to obtain a positive-factor
approximation in that setting. We first write the correlated Gaussian
exactly as a continuous mixture of isotropic Gaussian kernels, then
replace the continuous mixture by a finite positive quadrature. Under
the training priors considered below, the resulting kernels induce
transportable exponential--quadratic factors.
Appendix~\ref{app:component-approx-results} uses this representation.

\subsection{Continuous mixture and positive quadrature}

Let $0<\lambda<\lambda_{\min}(\Sigma)$ and
$U\sim\mathcal N(\mu,\Sigma-\lambda I)$. Adding independent
$\mathcal N(0,\lambda I)$ noise to $U$ gives
$\mathcal N(\mu,\Sigma)$, hence
\begin{equation}
\mathcal N(\theta;\mu,\Sigma)
=
\E_U\!\left[
\mathcal N(\theta;U,\lambda I)
\right].
\label{eq:continuous-prior-mixture}
\end{equation}
Eq.~\eqref{eq:continuous-prior-mixture} writes the correlated Gaussian exactly as a
continuous mixture of isotropic Gaussian kernels with common covariance
$\lambda I$.

The remaining task is to replace the continuous distribution over
centers $U$ by finitely many centers with positive weights. Write
\[
\Sigma-\lambda I
=
V\operatorname{diag}(s_1^2,\ldots,s_d^2)V^\top.
\]
Then
\[
U
=
\mu
+
V\operatorname{diag}(s_i)\xi,
\qquad
\xi\sim\mathcal N(0,I).
\]
For each coordinate $i$, choose a positive Gauss--Hermite rule with
nodes $\xi_{i,j}$ and weights $\eta_{i,j}$ for the standard-normal
expectation, normalized so that $\sum_j\eta_{i,j}=1$. The tensor node
indexed by $k=(j_1,\ldots,j_d)$ has
\begin{equation}
u_k
=
\mu
+
V\operatorname{diag}(s_i)
(\xi_{1,j_1},\ldots,\xi_{d,j_d})^\top,
\qquad
\omega_k
=
\prod_{i=1}^d\eta_{i,j_i}.
\label{eq:quadrature-nodes-weights}
\end{equation}
Every $\omega_k$ is positive and $\sum_k\omega_k=1$. Replacing the
continuous expectation by this quadrature gives
\begin{equation}
\widetilde\pi_{\rm new}(\theta)
=
\sum_k
\omega_k
\mathcal N(\theta;u_k,\lambda I).
\label{eq:quadrature-prior}
\end{equation}
Equation~\eqref{eq:continuous-prior-mixture} is exact. The finite
quadrature in Eq.~\eqref{eq:quadrature-prior} is the approximation.

The implementation uses the probabilists' Gauss--Hermite rule, whose
raw weights integrate against $e^{-\xi^2/2}$. It normalizes the weights
along each coordinate before forming the tensor products in
Eq.~\eqref{eq:quadrature-nodes-weights}. The quadrature order along
axis $i$ is
\begin{equation}
n_i
=
\min\!\left(
n_{\max},
\left\lfloor
h\frac{s_i}{\sqrt\lambda}
\right\rfloor
+2
\right),
\qquad
K
=
\prod_i n_i,
\end{equation}
where $h>0$ controls the node density and $n_{\max}$ caps the
per-axis order. Wider principal axes receive more nodes. For
$n_{\max}\ge2$ every axis receives at least two nodes, so $K\ge2^d$ grows
exponentially with dimension, and path-space weight estimation needs separate
proposal trajectories for each component. We use this construction only in
low-dimensional demonstrations.

The bandwidth, node orders, centers, and weights depend only on the
specified prior and construction settings, not on the observation or
posterior samples.

\subsection{From prior quadrature to prior-ratio coefficients}

The quadrature above approximates the target prior, but Spectra
requires a representation of the prior ratio. Dividing by the training
prior gives
\begin{equation}
\widetilde r(\theta)
=
\sum_k
\omega_k
\frac{
\mathcal N(\theta;u_k,\lambda I)
}{
\pi_{\rm tr}(\theta)
}.
\end{equation}

For the uniform training prior used in the Two Moons example,
$\pi_{\rm tr}$ is constant inside its support $B$. Up to the common
scale discussed in Appendix~\ref{app:factor-coefficients}, the
implementation therefore uses
\begin{equation}
b_k
=
\omega_k,
\qquad
\phi_k(\theta)
=
\mathcal N(\theta;u_k,\lambda I),
\end{equation}
with exponential--quadratic parameters
\begin{equation}
\kappa_k
=
\lambda^{-1},
\qquad
a_k
=
u_k/\lambda,
\qquad
c_k
=
-\frac{\lVert u_k\rVert^2}{2\lambda}
-\frac d2\log(2\pi\lambda).
\end{equation}
If the Gaussian mixture is restricted to $B$ and normalized as a
whole, the additional normalization constant is common to all factors
and therefore cancels from the component weights. No
observation-specific fitting of $b_k$ is performed.

For an isotropic Gaussian training prior, division by
$\pi_{\rm tr}$ instead follows the Gaussian-ratio construction in
Appendix~\ref{app:factor-coefficients}, with
$\mu_k=u_k$ and $v_k=\lambda$. An arbitrary training prior need not
preserve the isotropic factor family under division. Appendix~\ref{app:error-proof} 
separates the resulting prior-ratio approximation error from 
weight-estimation and component-sampling errors.

\section{Experimental protocols and reference validation}
\label{app:experiment-details}

This appendix specifies the benchmark tasks, prior shifts, evaluation and
sampling protocols, posterior references, and PriorGuide baselines used in the
experiments. Results are reported in
Appendices~\ref{app:k1-full}--\ref{app:component-approx-results}.

\subsection{Tasks and prior shifts}

For all six tasks in Section~\ref{sec:experiments}, the target posterior is
either available analytically or approximated by a reference procedure
validated in Appendix~\ref{app:references}. We follow the prior-shift
construction of PriorGuide \citep{yang2026priorguide}. For Gaussian training
priors, the marginal standard deviation of the target prior is $0.5$ times
that of the training prior for the mild shift and $0.2$ times that of the
training prior for the strong shift. For uniform training boxes, we apply the
same multipliers to the standard deviation of each uniform marginal.

Two Moons and SLCP use uniform training priors on $[-1,1]^d$, with $d=2$ and
$d=5$, respectively. For SLCP, this corresponds to rescaling the standard box
$[-3,3]^5$ by $1/3$. BCI uses a five-dimensional standard Gaussian, while the
Gaussian Linear tasks use $\mathcal N(0,0.1I)$ in 10 or 20 dimensions. OUP
uses a uniform prior on $[-1,1]^2$ after transforming its physical parameter
box $[0,2]\times[-2,2]$.

For each task, we train three Simformer backbones with training seeds 0, 1,
and 2 using the default architecture and training setup from the released
PriorGuide code. Each backbone is a six-layer transformer with four attention
heads and uses the variance-exploding schedule of
Appendix~\ref{app:sampling-settings}. Training uses 10,000 simulations from
the training prior, with 5\% held out for validation, a batch size of 1,000,
and 30,000 Adam steps. SLCP, which the released PriorGuide code does not
include, uses the same setup. No run triggered early stopping, so we use the final checkpoint in
every case. Across the six tasks, this gives 18 backbones, and every
experiment uses backbones from this set.

For each task and shift level, the single-factor benchmark uses one target
prior shared across ten test observations. For a training box $[l,h]$, a target
with standard deviation $s$ has its center drawn uniformly from
$[l+3s,h-3s]$ coordinatewise. For a Gaussian training prior with mean $\mu$
and marginal scale $s_{\rm tr}$, the center is drawn coordinatewise from
$[\mu-2s_{\rm tr}+3s,\mu+2s_{\rm tr}-3s]$. The mild and strong settings use
different center draws and separately simulated observations, so they differ in
more than prior width. Appendix~\ref{app:severity-sweep} isolates the effect
of width in a separate fixed-center sweep. Target priors are generated with
the released PriorGuide prior generator using a fixed seed. SLCP follows the
same uniform-prior recipe. Within each setting, all methods use the same prior,
observation, and frozen backbone.

For each task, the two-component benchmark uses ten mixture priors from the
same generator. Each prior contains two Gaussian components with the
strong-shift width, and their centers follow the single-factor rule above.
The first mixture coefficient is drawn uniformly from $[0.2,0.8]$, with the
second set to its complement. For each mixture prior, we draw three parameter
values from that prior and simulate one observation from each, yielding 30
prior--observation pairs per task.

\subsection{Metrics and statistical aggregation}

For each single-factor task and shift, we evaluate ten simulated observations
with three independently trained backbones. Base and the guided samplers use
three sampling seeds for each observation--backbone pair. We first average over
sampling seeds, then over backbones within each observation, and finally
summarize across the ten observations. Mixture results use the same hierarchy with 30
prior--observation pairs per task. Because three observations are generated
from each mixture prior, these 30 pairs are not 30 independent prior draws.
SIR uses one base-sample bank and one resampling draw for each
observation--backbone or pair--backbone combination.

C2ST \citep{lopezpaz2017c2st} assesses joint agreement between method and
reference samples, with ideal value $0.5$. MMTV averages total variation
distances across one-dimensional marginals, with ideal value zero. Both use
the metric implementations from the released PriorGuide code. For C2ST, a
random-forest classifier with five-fold cross-validation distinguishes 1,000
method samples from 1,000 samples drawn without replacement from the
10,000-sample reference bank. We repeat this over five reference subsamples
and average the classification accuracy. MMTV compares Gaussian kernel
density estimates of each marginal of the method samples and the full
reference bank. MMTV is undefined when a sample set has at most two distinct
points. This occurs only for SIR on the Gaussian Linear tasks. We also report
the root-mean-square error (RMSE) to the generating parameter as a
supplementary diagnostic. Because a multimodal posterior need not concentrate
around the generating parameter, RMSE is not used as a distributional
accuracy metric.

\subsection{Sampling and weight-estimation settings}
\label{app:sampling-settings}

\paragraph{Final posterior sampling.}
The main benchmark uses $(N,L)=(25,8)$ for PriorGuide and Spectra, except for
single-factor BCI at $(250,0)$. Base and SIR use $(100,0)$. Each method returns
1,000 posterior samples. The controlled comparisons in
Appendices~\ref{app:k1-full} and~\ref{app:k2-full} use $(100,0)$, while the
cost sweep in Appendix~\ref{app:compute-details} evaluates six sampler
configurations. The $(100,0)$ and $(25,8)$ settings differ in both the time
grid and the number of corrector updates.

\paragraph{Weight-estimation budgets.}
For the systematic mixture benchmarks, weight estimates are fixed across the
two output-sampling configurations. Path-space estimation uses 4,096
trajectories per component on a 400-point power grid and three estimator
seeds. The path-space estimator and the final posterior samplers use the
range-preserving grid
\begin{equation}
t_i
=
t_{\min}
+
(t_{\max}-t_{\min})u_i^2,
\qquad
u_i=\frac{i}{R-1},
\qquad
i=0,\ldots,R-1,
\label{eq:power-grid}
\end{equation}
which preserves both endpoints while placing more points at low noise.
All backbones use the same variance-exploding schedule,
\[
\sigma(t)
=
\sigma_{\min}
\left(
\frac{\sigma_{\max}}{\sigma_{\min}}
\right)^t,
\]
with $\sigma_{\min}=10^{-4}$, $\sigma_{\max}=15$,
$t_{\min}=10^{-5}$, and $t_{\max}=1$.

The direct estimator uses three independent base banks of 4,000 samples, each generated
with a 500-point uniform reverse grid. Neither weight estimator uses Langevin
corrector updates. For path-space estimation, the proposal transitions enter
the likelihood ratio in Eq.~\eqref{eq:path-rn}. Corrector updates are used
only by the final output sampler when enabled. Each estimator seed $s$ produces one log-normalizer estimate
$\log\widehat Z_{k,s}$ per component. We average these values across the three
estimator seeds and compute
\begin{equation}
\overline{\log\widehat Z_k}
=
\frac{1}{S}
\sum_{s=1}^{S}
\log\widehat Z_{k,s},
\qquad
\widehat\alpha
=
\operatorname{softmax}
\left(
\log b+\overline{\log\widehat Z}
\right),
\qquad
S=3.
\end{equation}
We use this aggregation for all component counts.

\paragraph{Transformed noise levels.}
Because $\sigma$ is log-linear in $t$, the transformed noise level
$\rho_{k,\tau}$ in Eq.~\eqref{eq:transform-pieces} corresponds to
\begin{equation}
t_{\rho}
=
t
-
\frac{\log D_{k,\tau}}
{2\log(\sigma_{\max}/\sigma_{\min})}.
\label{eq:time-shift}
\end{equation}
We query the frozen model directly at $t_{\rho}$ without clipping. Clipping
would change the queried noise level and invalidate the identity in
Eq.~\eqref{eq:exact-score}. In all tasks, shifts, and sampler configurations
reported here, $t_{\rho}$ remains within $[t_{\min},t_{\max}]$.

\paragraph{Terminal initialization.}
\label{app:terminal-init}

All Spectra components and matched guided baselines use the frozen base
model's terminal Gaussian. Let $\mu_{\rm train}$ and $v_{\rm train}$ denote
the empirical mean and coordinatewise variances of the training parameters
in model coordinates. At the largest noise level
$\tau_T=\sigma(t_{\max})^2$, we initialize
\begin{equation}
z_T
\sim
\mathcal N
\left(
\mu_{\rm train},
\operatorname{diag}(v_{\rm train})+\tau_T I
\right).
\label{eq:terminal-init}
\end{equation}
For a fixed backbone, this initialization is shared across observations and
component labels. The selected component changes only the score used during
reverse sampling.

The ideal component sampler in Appendix~\ref{app:mixture-proof} would instead
initialize from $q_{k,\tau_T}(\cdot\mid x)$. The shared Gaussian in
Eq.~\eqref{eq:terminal-init} is a high-noise approximation to this
component-specific terminal distribution and is generally not identical to it
at finite $\tau_T$. This approximation is separate from the exact score
identity and contributes to the component-sampling error in
Appendix~\ref{app:error-proof}, together with learned-score and numerical
integration errors.

For the path-space estimator, the base and component chains use the same
initialization, so the terminal-density term in Eq.~\eqref{eq:path-rn} is
identically zero in our experiments. Appendix~\ref{app:path} retains this
term for the general case of different terminal distributions.

\subsection{Posterior references and support checks}
\label{app:references}

Posterior references are analytic for the Gaussian Linear tasks and constructed numerically for the remaining tasks. 
Table~\ref{tab:references} summarizes each construction and its validation. 
Every reference bank contains 10,000 samples. 
SLCP and BCI use tempered sequential Monte Carlo (SMC), following \citet{delmoral2006smc}. 
For BCI, lower Gauss--Hermite orders measurably change the posterior, whereas orders 301 and 601 agree within the tested target region.

\begin{table}[ht]
\centering
\small
\caption{\textbf{Posterior references and validation.}
Reported validation statistics are worst-case values over the evaluated
single-factor observations; BCI also includes the mixture pairs.
Cross-run C2ST compares independent reference runs of the same posterior
(ideal value $0.5$).}
\label{tab:references}
\begin{tabular}{@{}>{\raggedright\arraybackslash}p{.15\linewidth}>{\raggedright\arraybackslash}p{.31\linewidth}>{\raggedright\arraybackslash}p{.48\linewidth}@{}}
\toprule
Task & Reference & Validation \\
\midrule
GL-10D/20D & closed-form posterior & analytic \\
Two Moons & exact-likelihood quadrature on a $4096\times4096$ grid & grid refinement changes the log normalizer by at most $2\times10^{-4}$ nats \\
OUP & exact-likelihood quadrature on a $4096\times4096$ grid & grid refinement changes the log normalizer by less than $10^{-12}$ nats; target mass on the grid boundary at most $4\times10^{-10}$ (both shifts) \\
SLCP & tempered SMC, 40,000 particles, four independent runs & cross-run C2ST at most $0.509$ (strong) and $0.521$ (mild) \\
BCI & tempered SMC, 8,000 particles, four independent runs; Gauss--Hermite likelihood of order 301 & cross-run C2ST at most $0.509$ (both shifts and mixtures); quadrature-order check \\
\bottomrule
\end{tabular}
\end{table}
\paragraph{Reference component weights.}
For the two-component benchmark, reference weights are obtained by averaging
the responsibilities
$b_k\phi_k(\theta)/\sum_j b_j\phi_j(\theta)$
under the target-posterior reference. For Two Moons and OUP, we evaluate this
expectation on the quadrature grid restricted to the training box. Because the
base posterior is supported there, this gives the weights in
Eq.~\eqref{eq:posterior-mixture}. For SLCP and BCI, we average over their SMC
reference samples. The Gaussian Linear weights are available in closed form
from the conjugate posterior. Appendices~\ref{app:k-composition}
and~\ref{app:component-approx-results} describe the reference weights used in
the remaining mixture studies.

\paragraph{Spectra variants.}
The mixture results in Appendices~\ref{app:k2-full}
and~\ref{app:component-approx-results} report Spectra with three sources of
component weights. Spectra-PS, the method in the main text, uses the path-space
estimator of Appendix~\ref{app:path}, and Spectra-Direct uses the direct
estimator of Appendix~\ref{app:direct-estimation}. Spectra-Ref uses the
reference weights above. It needs the reference posterior and is therefore not
a practical method, but it separates weight-estimation error from the rest of
the sampler.

\paragraph{Reference support.}
For tasks with uniform training support, observation parameters are drawn from
the target prior restricted to the training box. Two Moons and OUP truncate
each Gaussian component separately while retaining the mixture coefficients,
whereas SLCP truncates the mixture as a whole. The two constructions coincide
for a single component, and all component centers lie at least three component
standard deviations inside the box. Spectra does not extend the support of the
training prior.

The benchmark metrics nevertheless use the original untruncated Gaussian
reference. To check the effect of this mismatch, we reevaluate the
single-factor Two Moons, SLCP, and OUP experiments against references
restricted to the training support. Across Base, SIR, PriorGuide, and Spectra,
all 12 C2ST/MMTV method orderings are unchanged, and the largest change in an
aggregate metric is $0.0013$. At least $98.69\%$ of every reference bank lies
inside the training support. This support mismatch therefore does not affect the reported method comparisons in these settings.

\subsection{PriorGuide protocols}
\label{app:priorguide-variants}

\paragraph{Exact-ratio comparison.}
The main comparison supplies PriorGuide with the same exact analytic
prior ratio as Spectra, bypassing Gaussian mixture model fitting wherever it
would otherwise be required. The Gaussian reverse-conditional approximation
remains part of PriorGuide's guidance. Both methods use the sampling settings
in Appendix~\ref{app:sampling-settings}, following
\citet[Appendix C.3.2]{yang2026priorguide}: $(25,8)$ for the uniform-prior and
Gaussian Linear tasks and $(250,0)$ for BCI. We use the BCI setting for
single-factor shifts only, and BCI mixtures use $(25,8)$. This removes prior-ratio fitting
error from the comparison while retaining the difference between the two
score updates.

\paragraph{Langevin update convention.}
All matched guided methods use the corrector implemented in the released
PriorGuide code. For a positive step magnitude $\Delta t$,
\[
\delta=\frac{\eta g(t)^2\Delta t}{2},
\qquad
z\leftarrow z+\delta\widehat s_q(z,t,x)
+\sqrt{2\delta+10^{-8}}\,\epsilon,
\qquad
\epsilon\sim\mathcal N(0,I),
\quad \eta=0.5,
\]
where $g(t)^2=2\sigma(t)\sigma'(t)$. PriorGuide's Eq.~(A1) gives a step half
this size under the same convention. We follow the executable update for all
methods. For Spectra, $g(t)$ is evaluated at the current sampler time $t$,
while the score itself is queried at the transformed noise level.

\paragraph{Vector--Jacobian-product implementation.}
PriorGuide's guidance uses the denoiser Jacobian only through a product with a
vector, so the same update can be evaluated with a reverse-mode
vector--Jacobian product (VJP) instead of explicitly forming the full
parameter--parameter Jacobian. We use this implementation, denoted PriorGuide (VJP), 
as the strengthened timing baseline in Appendix~\ref{app:compute-details}. 
It changes only the implementation of the guidance update.

At $(N,L)=(25,8)$, a paired validation on Two Moons, SLCP, and GL-20D gives
absolute mean C2ST differences of at most $5\times10^{-4}$ and MMTV differences
below $10^{-4}$ between PriorGuide (VJP) and the benchmarked implementation,
with all 95\% paired intervals containing zero. Accuracy tables therefore use
the benchmarked implementation, while timing comparisons include the stronger
VJP implementation. \pgfullcov{} does not admit the same reduction because
its Tweedie covariance requires the full Jacobian.

\subsection{PriorGuide-FullCov}
\label{app:pg-fullcov}

PriorGuide-FullCov (\pgfullcov{}) is an additional moment-matched baseline constructed in
this work, rather than a variant proposed by \citet{yang2026priorguide}. It
tests whether PriorGuide's scalar covariance approximation explains the
difference from exact transport. We retain a Gaussian approximation to the
reverse conditional but replace its scalar covariance with a
moment-consistent full covariance, following the Tweedie moment-projection
approach of \citet{boys2024tmpd}.

\paragraph{Moment identities.}
For fixed $x$, suppress the observation argument and write
$\mu_{p,\tau}(z)=\mathbb E_p[\theta\mid z,x]$. The Gaussian-channel moment
identities underlying Tweedie's formula \citep[Section~2]{efron2011tweedie}
give
\begin{equation}
\mu_{p,\tau}(z)
=
z+\tau s_p(z,\tau),
\qquad
\Sigma_{p,\tau}(z)
=
\operatorname{Cov}_p(\theta\mid z,x)
=
\tau\nabla_z\mu_{p,\tau}(z)
=
\tau I+\tau^2\nabla_z s_p(z,\tau).
\label{eq:fullcov-moments}
\end{equation}
These are exact moment identities for the exact base score. They do not imply
that the reverse conditional itself is Gaussian.

In the implementation, we substitute the learned score and repair the
resulting covariance before using it in the Gaussian surrogate. Let
$J=\nabla_z\mu_{p,\tau}(z)$ denote the learned denoiser Jacobian on the
parameter block, so that $\Sigma_{p,\tau}=\tau J$. We symmetrize $J$ and floor
its eigenvalues,
\begin{equation}
J_{\rm sym}
=
\tfrac12(J+J^\top)
=
U\operatorname{diag}(\lambda)U^\top,
\qquad
J_{\rm psd}
=
U\operatorname{diag}\!\left(\max(\lambda,\varepsilon)\right)U^\top,
\qquad
\varepsilon=10^{-6},
\label{eq:psd-repair}
\end{equation}
and use $C=\tau J_{\rm psd}$ as the local covariance.

\paragraph{Gaussian twist.}
Let $m$ and $C$ denote the resulting local mean and covariance. For a supplied
ratio
\[
r(\theta)
\propto
\sum_k w_k\mathcal N(\theta;\mu_k,\Sigma_k),
\qquad
w_k>0,
\]
multiplying the surrogate $\mathcal N(\theta;m,C)$ by component $k$ gives
\begin{equation}
\widetilde\mu_k
=
m+C(C+\Sigma_k)^{-1}(\mu_k-m),
\qquad
\ell_k
=
\mathcal N(\mu_k;m,C+\Sigma_k).
\end{equation}
The corresponding twisted mean and guidance are
\begin{equation}
\beta_k
=
\frac{w_k\ell_k}{\sum_j w_j\ell_j},
\qquad
\widetilde m_q
=
\sum_k\beta_k\widetilde\mu_k,
\qquad
g_{\mathrm{FC}}
=
\frac{\widetilde m_q-m}{\tau}.
\end{equation}
Here $\beta_k$ are noisy-state-dependent weights within the Gaussian
surrogate, not Spectra's observation-level component weights. The guidance
substitutes this Gaussian-product mean into the identity
$s_q-s_p=(m_q-m_p)/\tau$.

For a single isotropic Gaussian ratio, this construction reduces, before
covariance repair, to the TMPD update of
\citet[Eq.~(11)]{boys2024tmpd} with an identity observation operator under VE
corruption. Computing $C$ requires a score Jacobian at every guidance
evaluation, including corrector updates when enabled. Both PriorGuide and
\pgfullcov{} receive the exact analytic prior ratio in these controlled
comparisons.

\section{Single-factor results}
\label{app:k1-full}

This appendix provides supporting single-factor results, including a
controlled $(N,L)=(100,0)$ comparison with \pgfullcov{}, paired and overlap
diagnostics, the practical sampling configuration, and the fixed-center
severity sweep behind Figure~\ref{fig:severity-sweep}.

\begin{table}[!htbp]
\centering
\caption{
\textbf{Single-factor comparison at $(N,L)=(100,0)$.}
Mean (\textcolor[gray]{0.48}{standard deviation}) across ten observations,
after averaging sampling seeds within each backbone and then
averaging the three backbones within each observation.
C2ST is ideal at $0.5$ and MMTV at zero.
Bold indicates the displayed mean closest to the ideal value within each
task and shift, including rounded ties. RMSE to the generating parameter is a
supplementary diagnostic and is not bolded.
}
\label{tab:k1-controlled}
\begingroup
\fontsize{8}{9.4}\selectfont
\setlength{\tabcolsep}{2pt}
\renewcommand{\arraystretch}{1.06}
\begin{tabular*}{\linewidth}{@{\extracolsep{\fill}}llrrrrrr@{}}
\toprule
& & \multicolumn{3}{c}{Mild} & \multicolumn{3}{c}{Strong} \\
\cmidrule(lr){3-5}\cmidrule(l){6-8}
Task & Method & C2ST & MMTV & RMSE & C2ST & MMTV & RMSE \\
\midrule
\multirow{5}{*}{Two Moons} & Base & \benchval{0.568}{0.036} & \benchval{0.219}{0.087} & \benchval{0.397}{0.197} & \benchval{0.691}{0.095} & \benchval{0.426}{0.192} & \benchval{0.174}{0.106} \\
 & SIR & \benchval{0.525}{0.005} & \benchval{0.075}{0.011} & \benchval{0.340}{0.147} & \benchval{0.544}{0.018} & \benchval{0.070}{0.012} & \benchval{0.100}{0.061} \\
 & PriorGuide (exact ratio) & \benchval{0.522}{0.010} & \benchval{0.105}{0.040} & \benchval{0.371}{0.173} & \benchval{0.554}{0.032} & \benchval{0.167}{0.085} & \benchval{0.136}{0.063} \\
 & PG-FullCov & \benchval{0.528}{0.015} & \benchval{0.123}{0.057} & \benchval{0.379}{0.180} & \benchval{0.552}{0.038} & \benchval{0.159}{0.093} & \benchval{0.142}{0.070} \\
 & Spectra & \benchval{\textbf{0.512}}{0.006} & \benchval{\textbf{0.072}}{0.007} & \benchval{0.346}{0.151} & \benchval{\textbf{0.519}}{0.009} & \benchval{\textbf{0.068}}{0.012} & \benchval{0.102}{0.059} \\
\midrule
\multirow{5}{*}{SLCP} & Base & \benchval{0.765}{0.057} & \benchval{0.255}{0.041} & \benchval{0.324}{0.089} & \benchval{0.917}{0.019} & \benchval{0.466}{0.072} & \benchval{0.312}{0.060} \\
 & SIR & \benchval{0.640}{0.086} & \benchval{0.123}{0.054} & \benchval{0.256}{0.088} & \benchval{0.713}{0.067} & \benchval{0.101}{0.043} & \benchval{0.097}{0.023} \\
 & PriorGuide (exact ratio) & \benchval{0.617}{0.091} & \benchval{0.121}{0.055} & \benchval{0.257}{0.085} & \benchval{0.635}{0.046} & \benchval{0.127}{0.027} & \benchval{0.127}{0.027} \\
 & PG-FullCov & \benchval{0.618}{0.091} & \benchval{0.121}{0.057} & \benchval{0.254}{0.084} & \benchval{0.618}{0.054} & \benchval{0.112}{0.033} & \benchval{0.123}{0.027} \\
 & Spectra & \benchval{\textbf{0.615}}{0.092} & \benchval{\textbf{0.116}}{0.057} & \benchval{0.251}{0.089} & \benchval{\textbf{0.587}}{0.067} & \benchval{\textbf{0.097}}{0.044} & \benchval{0.096}{0.022} \\
\midrule
\multirow{5}{*}{OUP} & Base & \benchval{0.587}{0.064} & \benchval{0.186}{0.075} & \benchval{0.196}{0.066} & \benchval{0.742}{0.088} & \benchval{0.397}{0.079} & \benchval{0.221}{0.079} \\
 & SIR & \benchval{0.524}{0.009} & \benchval{0.054}{0.009} & \benchval{0.184}{0.075} & \benchval{0.548}{0.050} & \benchval{\textbf{0.045}}{0.011} & \benchval{0.127}{0.036} \\
 & PriorGuide (exact ratio) & \benchval{0.509}{0.005} & \benchval{0.058}{0.011} & \benchval{0.187}{0.074} & \benchval{0.521}{0.023} & \benchval{0.076}{0.033} & \benchval{0.131}{0.037} \\
 & PG-FullCov & \benchval{0.509}{0.004} & \benchval{0.057}{0.009} & \benchval{0.185}{0.074} & \benchval{0.507}{0.005} & \benchval{0.048}{0.012} & \benchval{0.129}{0.037} \\
 & Spectra & \benchval{\textbf{0.508}}{0.003} & \benchval{\textbf{0.052}}{0.008} & \benchval{0.184}{0.075} & \benchval{\textbf{0.506}}{0.006} & \benchval{0.047}{0.011} & \benchval{0.129}{0.037} \\
\midrule
\multirow{5}{*}{BCI} & Base & \benchval{0.758}{0.030} & \benchval{0.254}{0.049} & \benchval{0.684}{0.058} & \benchval{0.948}{0.016} & \benchval{0.497}{0.045} & \benchval{0.714}{0.026} \\
 & SIR & \benchval{0.591}{0.026} & \benchval{0.086}{0.031} & \benchval{0.450}{0.061} & \benchval{0.799}{0.066} & \benchval{0.224}{0.312} & \benchval{0.222}{0.027} \\
 & PriorGuide (exact ratio) & \benchval{0.568}{0.024} & \benchval{0.085}{0.028} & \benchval{0.448}{0.060} & \benchval{0.572}{0.015} & \benchval{0.079}{0.010} & \benchval{0.224}{0.027} \\
 & PG-FullCov & \benchval{\textbf{0.564}}{0.022} & \benchval{\textbf{0.083}}{0.027} & \benchval{0.451}{0.061} & \benchval{0.552}{0.015} & \benchval{\textbf{0.065}}{0.006} & \benchval{0.222}{0.026} \\
 & Spectra & \benchval{0.567}{0.024} & \benchval{0.085}{0.028} & \benchval{0.449}{0.059} & \benchval{\textbf{0.551}}{0.015} & \benchval{0.066}{0.007} & \benchval{0.222}{0.026} \\
\midrule
\multirow{5}{*}{GL-10D} & Base & \benchval{0.898}{0.024} & \benchval{0.308}{0.028} & \benchval{0.298}{0.024} & \benchval{0.997}{0.002} & \benchval{0.664}{0.038} & \benchval{0.328}{0.033} \\
 & SIR & \benchval{0.760}{0.056} & \benchval{0.225}{0.260} & \benchval{0.212}{0.024} & \benchval{0.998}{0.002} & \benchval{0.965}{0.064} & \benchval{0.123}{0.014} \\
 & PriorGuide (exact ratio) & \benchval{0.548}{0.015} & \benchval{0.055}{0.004} & \benchval{0.216}{0.023} & \benchval{0.558}{0.029} & \benchval{0.058}{0.010} & \benchval{0.095}{0.011} \\
 & PG-FullCov & \benchval{\textbf{0.531}}{0.015} & \benchval{\textbf{0.043}}{0.005} & \benchval{0.212}{0.024} & \benchval{\textbf{0.522}}{0.015} & \benchval{\textbf{0.040}}{0.005} & \benchval{0.093}{0.011} \\
 & Spectra & \benchval{0.536}{0.017} & \benchval{0.045}{0.005} & \benchval{0.212}{0.024} & \benchval{0.525}{0.015} & \benchval{\textbf{0.040}}{0.004} & \benchval{0.093}{0.011} \\
\midrule
\multirow{5}{*}{GL-20D} & Base & \benchval{0.947}{0.016} & \benchval{0.300}{0.017} & \benchval{0.289}{0.019} & \benchval{1.000}{0.000} & \benchval{0.645}{0.017} & \benchval{0.303}{0.018} \\
 & SIR & \benchval{0.927}{0.029} & \benchval{0.888}{0.061} & \benchval{0.208}{0.017} & \benchval{1.000}{0.000} & \benchval{0.852}{0.139}\rlap{\textsuperscript{$\dagger$}} & \benchval{0.149}{0.013} \\
 & PriorGuide (exact ratio) & \benchval{0.584}{0.016} & \benchval{0.063}{0.009} & \benchval{0.213}{0.016} & \benchval{0.580}{0.019} & \benchval{0.057}{0.005} & \benchval{0.093}{0.007} \\
 & PG-FullCov & \benchval{\textbf{0.546}}{0.017} & \benchval{\textbf{0.042}}{0.003} & \benchval{0.206}{0.016} & \benchval{\textbf{0.528}}{0.017} & \benchval{\textbf{0.037}}{0.003} & \benchval{0.091}{0.007} \\
 & Spectra & \benchval{0.554}{0.019} & \benchval{0.044}{0.003} & \benchval{0.206}{0.016} & \benchval{0.531}{0.018} & \benchval{0.038}{0.003} & \benchval{0.090}{0.007} \\
\bottomrule
\end{tabular*}

\par\vspace{3pt}
\raggedright
\fontsize{7.5}{8.8}\selectfont
\textdagger{} Strong-shift SIR MMTV is defined for six of the ten GL-20D
observations. Degenerate resampling prevents evaluation on the remaining four,
so the reported mean and standard deviation use the six valid observations.
\par
\endgroup
\end{table}

\subsection{Controlled comparison without corrector updates}
\label{app:k1-controlled}

Table~\ref{tab:k1-controlled} compares all six tasks at
$(N,L)=(100,0)$, without Langevin corrector updates. Under strong shifts,
Spectra improves on \pgfullcov{} on Two Moons and SLCP. The two methods are
similar on OUP and BCI, while \pgfullcov{} has a small C2ST advantage on the
Gaussian Linear controls. Under mild shifts, Spectra is again better on Two
Moons, whereas \pgfullcov{} attains slightly lower C2ST on BCI and both
Gaussian Linear tasks, by at most $0.008$ (Table~\ref{tab:k1-paired}).
Appendix~\ref{app:compute-details} reports the computational cost of this
additional baseline.

\FloatBarrier
\subsection{Paired comparisons and overlap diagnostics}

Table~\ref{tab:k1-paired} reports observation-level paired C2ST differences for all six tasks under the controlled configuration of Table~\ref{tab:k1-controlled}.

\paragraph{Clean-space overlap.}
For SIR, we measure overlap using the effective sample size 
$(\sum_i w_i)^2/\sum_i w_i^2$ of the prior-ratio weights
$w_i=r(\theta_i)$ on the 20,000-sample Base bank, normalized by the bank
size. Under strong shifts, the median ESS fractions across the 30
observation--backbone pairs are approximately $45\%$, $6.0\%$, and $1.8\%$ on Two
Moons, SLCP, and BCI, but only $8\times10^{-5}$ and $5\times10^{-5}$ on
GL-10D and GL-20D. The latter correspond to roughly two effective samples or
fewer from the full bank. Under mild shifts, the corresponding fractions are
$79\%$, $42\%$, and $34\%$ on the first three tasks, compared with $1.9\%$
and $0.14\%$ on the Gaussian Linear tasks. The collapse in overlap on the
Gaussian Linear tasks is consistent with the severe SIR degeneration in
Table~\ref{tab:k1-controlled}.

\begin{table}[!htbp]
\centering
\scriptsize
\caption{
\textbf{Paired C2ST differences at $(N,L)=(100,0)$.}
Entries are baseline minus Spectra, with 95\% paired Student-$t$ confidence
intervals across ten observations. Positive values correspond to lower C2ST for
Spectra.
}
\label{tab:k1-paired}
\begin{tabular}{llcc}
\toprule
Shift & Task & PriorGuide $-$ Spectra & \pgfullcov{} $-$ Spectra \\
\midrule
\multirow{6}{*}{Mild}
& Two Moons & $+0.0097\ [+0.0018,+0.0177]$ & $+0.0151\ [+0.0039,+0.0264]$ \\
& SLCP      & $+0.0019\ [-0.0070,+0.0108]$ & $+0.0031\ [-0.0026,+0.0087]$ \\
& OUP       & $+0.0014\ [-0.0025,+0.0052]$ & $+0.0011\ [-0.0011,+0.0034]$ \\
& BCI       & $+0.0009\ [-0.0022,+0.0040]$ & $-0.0026\ [-0.0047,-0.0004]$ \\
& GL-10D    & $+0.0113\ [+0.0044,+0.0182]$ & $-0.0049\ [-0.0074,-0.0024]$ \\
& GL-20D    & $+0.0309\ [+0.0130,+0.0488]$ & $-0.0073\ [-0.0109,-0.0038]$ \\
\midrule
\multirow{6}{*}{Strong}
& Two Moons & $+0.0346\ [+0.0155,+0.0537]$ & $+0.0330\ [+0.0097,+0.0562]$ \\
& SLCP      & $+0.0476\ [+0.0260,+0.0692]$ & $+0.0310\ [+0.0128,+0.0491]$ \\
& OUP       & $+0.0156\ [+0.0014,+0.0297]$ & $+0.0007\ [-0.0008,+0.0021]$ \\
& BCI       & $+0.0214\ [+0.0128,+0.0300]$ & $+0.0012\ [-0.0023,+0.0047]$ \\
& GL-10D    & $+0.0339\ [+0.0216,+0.0461]$ & $-0.0021\ [-0.0036,-0.0007]$ \\
& GL-20D    & $+0.0492\ [+0.0382,+0.0602]$ & $-0.0030\ [-0.0051,-0.0009]$ \\
\bottomrule
\end{tabular}

\end{table}

\FloatBarrier
\subsection{Practical sampling configurations}
\label{app:k1-practical}

\paragraph{OUP under the practical configuration.}
PriorGuide and Spectra have similar OUP accuracy at $(N,L)=(25,8)$. Under
the strong shift, moving from $(100,0)$ to $(25,8)$ improves PriorGuide's
C2ST from $0.521$ to $0.505$, while Spectra changes from $0.506$ to $0.505$.
The practical configuration therefore removes the small C2ST gap seen in the
controlled comparison. Under the mild shift, the remaining differences are
also small, with C2ST differing by $0.0013$ and MMTV by $0.0007$ in opposite
directions.

\paragraph{Held-out evaluation with corrector updates.}
Following the prespecified selection rule, we froze $(25,8)$ and the
$(100,0)$ control after development on three observations with one
backbone. Table~\ref{tab:k1-heldout} evaluates the remaining seven observations
at $(25,8)$ using three independently trained backbones and three sampling seeds.
On Two Moons and SLCP, Spectra has lower C2ST and MMTV than both PriorGuide
and \pgfullcov{}, with the corresponding 95\% paired Student-$t$ intervals
excluding zero. On GL-20D, the C2ST intervals against both baselines include
zero. These seven observations are the held-out subset of the ten-observation
practical benchmark, so this is a configuration-selection check rather than
an independent replication.

\begin{table}[!htbp]
\centering
\caption{
\textbf{Held-out single-factor comparison with Langevin corrections.}
Strong shifts at $(N,L)=(25,8)$. Mean
(\textcolor[gray]{0.48}{standard deviation}) across seven held-out
observations after averaging three sampling seeds within each of three
backbones. C2ST is ideal at $0.5$ and MMTV at zero. Bold indicates the
displayed mean closest to the ideal value within each task, including rounded
ties.
}
\label{tab:k1-heldout}
\begingroup
\fontsize{8}{9.4}\selectfont
\setlength{\tabcolsep}{3pt}
\begin{tabular}{lrrrrrr}
\toprule
& \multicolumn{2}{c}{Two Moons} & \multicolumn{2}{c}{SLCP} & \multicolumn{2}{c}{GL-20D} \\
\cmidrule(lr){2-3}\cmidrule(lr){4-5}\cmidrule(lr){6-7}
Method & C2ST & MMTV & C2ST & MMTV & C2ST & MMTV \\
\midrule
PriorGuide (exact ratio) & \benchval{0.560}{0.039} & \benchval{0.174}{0.099} & \benchval{0.617}{0.029} & \benchval{0.115}{0.030} & \benchval{0.531}{0.025} & \benchval{\textbf{0.035}}{0.004} \\
PG-FullCov & \benchval{0.567}{0.047} & \benchval{0.192}{0.118} & \benchval{0.620}{0.030} & \benchval{0.117}{0.032} & \benchval{\textbf{0.530}}{0.023} & \benchval{\textbf{0.035}}{0.004} \\
Spectra & \benchval{\textbf{0.534}}{0.017} & \benchval{\textbf{0.103}}{0.033} & \benchval{\textbf{0.586}}{0.038} & \benchval{\textbf{0.100}}{0.041} & \benchval{0.531}{0.024} & \benchval{\textbf{0.035}}{0.004} \\
\bottomrule
\end{tabular}

\endgroup
\end{table}
\FloatBarrier

\subsection{Fixed-center shift-severity sweep}
\label{app:severity-sweep}

Because the mild and strong benchmarks differ in prior center and observations
as well as width, they do not isolate the effect of shift severity. The sweep
in Figure~\ref{fig:severity-sweep} therefore varies prior width while holding
the other experimental factors fixed. For Two Moons and OUP, we fix the prior
center, the ten observations from the mild single-factor setting, the three
backbones, and the sampler, and set the target-prior marginal standard
deviation to $0.50$, $0.40$, $0.30$, $0.25$, and $0.20$ times that of the
training prior. PriorGuide receives the exact analytic ratio, and both guided
methods use $(N,L)=(25,8)$ with three sampling seeds. All runs are evaluated
against references restricted to the training support
(Appendix~\ref{app:references}), whereas Table~\ref{tab:k1-results} uses the
full benchmark reference. The two sets of results are therefore not compared
numerically.

On Two Moons, the gap grows systematically as the target prior narrows in
both metrics, and the paired difference is positive for at least nine of ten
observations at every width. On OUP, all differences remain below $0.002$ in
magnitude and show no systematic trend.
\begin{table}[!htbp]
\centering
\small
\caption{
\textbf{Fixed-center shift-severity sweep.}
Entries are mean paired differences (PriorGuide minus Spectra) across ten
observations, with 95\% paired Student-$t$ confidence intervals after
averaging over three sampling seeds and three backbones within each
observation. Width is the target-prior marginal standard deviation relative
to the training prior. Positive differences favor Spectra. Favor gives the
number of observations with a positive paired difference.
}
\label{tab:severity-sweep}
\begin{tabular*}{\linewidth}{@{\extracolsep{\fill}}llrrrr@{}}
\toprule
 & & \multicolumn{2}{c}{C2ST} & \multicolumn{2}{c}{MMTV} \\
\cmidrule(lr){3-4}\cmidrule(l){5-6}
Task & Width & Difference [95\% CI] & Favor & Difference [95\% CI] & Favor \\
\midrule
\multirow{5}{*}{Two Moons} & 0.50 & $+$0.0075 [$+$0.0006, $+$0.0144] & 10/10 & $+$0.0297 [$+$0.0062, $+$0.0532] & 10/10 \\
 & 0.40 & $+$0.0191 [$+$0.0032, $+$0.0350] & 9/10 & $+$0.0596 [$+$0.0203, $+$0.0990] & 10/10 \\
 & 0.30 & $+$0.0434 [$+$0.0179, $+$0.0688] & 10/10 & $+$0.1140 [$+$0.0532, $+$0.1748] & 10/10 \\
 & 0.25 & $+$0.0571 [$+$0.0280, $+$0.0861] & 9/10 & $+$0.1486 [$+$0.0788, $+$0.2185] & 10/10 \\
 & 0.20 & $+$0.0727 [$+$0.0414, $+$0.1039] & 10/10 & $+$0.1828 [$+$0.1065, $+$0.2590] & 10/10 \\
\midrule
\multirow{5}{*}{OUP} & 0.50 & $-$0.0013 [$-$0.0027, $+$0.0001] & 3/10 & $+$0.0008 [$+$0.0001, $+$0.0015] & 8/10 \\
 & 0.40 & $-$0.0003 [$-$0.0015, $+$0.0009] & 4/10 & $+$0.0009 [$-$0.0003, $+$0.0021] & 6/10 \\
 & 0.30 & $+$0.0003 [$-$0.0007, $+$0.0012] & 5/10 & $+$0.0011 [$-$0.0003, $+$0.0025] & 7/10 \\
 & 0.25 & 0.0000 [$-$0.0025, $+$0.0025] & 5/10 & $+$0.0013 [$-$0.0001, $+$0.0027] & 8/10 \\
 & 0.20 & $-$0.0002 [$-$0.0024, $+$0.0019] & 4/10 & $+$0.0004 [$-$0.0022, $+$0.0029] & 7/10 \\
\bottomrule
\end{tabular*}

\end{table}

\FloatBarrier
\section{Mixture results}
\label{app:k2-full}

This appendix reports the systematic two-component benchmark, examines
posterior component allocation and weight estimation, and evaluates
composition beyond two components.

\subsection{Systematic mixture benchmark}
\label{app:k2-systematic}

The systematic benchmark evaluates ten mixture priors with three observations
per prior on each of six tasks (Appendix~\ref{app:experiment-details}).
Table~\ref{tab:k2-controlled} reports the controlled comparison at
$(N,L)=(100,0)$. At the practical sampling configurations of
Table~\ref{tab:k2-results}, \direct{} and \pathspace{} use the same fixed weight estimates 
for both output configurations and share the same final component sampler. 
The two give nearly identical results on Two Moons, SLCP, OUP, and BCI.
On the Gaussian Linear tasks, \pathspace{}
gives lower mean C2ST and MMTV, with the larger difference on
GL-10D.

\begin{table}[htbp]
\centering
\caption{
\textbf{Mixture posterior inference across 180 prior--observation pairs.}
Mean (\textcolor[gray]{0.48}{standard deviation}) across 30 pairs per task,
after averaging sampling seeds within each backbone and then the three
backbones within each pair. All output samplers use $(N,L)=(100,0)$.
PriorGuide receives the exact prior ratio. Direct, PS, and Ref are the Spectra variants of Appendix~\ref{app:references}.
Ref uses reference component weights and is included only as a diagnostic.
C2ST is ideal at $0.5$ and MMTV at zero. Bold indicates the displayed mean closest 
to the ideal value among methods that do not use reference weights, 
including rounded ties.
}
\label{tab:k2-controlled}
\begingroup
\fontsize{8}{9.4}\selectfont
\setlength{\tabcolsep}{2pt}
\renewcommand{\arraystretch}{1.06}
\begin{tabular*}{\linewidth}{@{\extracolsep{\fill}}lrrrrrrr@{}}
\toprule
 & & & & & \multicolumn{3}{c}{Spectra} \\
\cmidrule(l){6-8}
Task & Base & SIR & PriorGuide & PG-FullCov & Direct & PS & Ref \\
\midrule
\multicolumn{8}{c}{\textbf{C2ST}} \\
\midrule
Two Moons & \benchval{0.736}{0.074} & \benchval{0.543}{0.018} & \benchval{0.552}{0.035} & \benchval{0.538}{0.021} & \benchval{\textbf{0.514}}{0.007} & \benchval{0.515}{0.007} & \benchval{0.515}{0.008} \\
SLCP & \benchval{0.946}{0.029} & \benchval{0.796}{0.086} & \benchval{0.611}{0.043} & \benchval{0.587}{0.042} & \benchval{\textbf{0.554}}{0.040} & \benchval{0.555}{0.040} & \benchval{0.554}{0.040} \\
OUP & \benchval{0.703}{0.063} & \benchval{0.529}{0.011} & \benchval{0.517}{0.009} & \benchval{0.507}{0.008} & \benchval{\textbf{0.506}}{0.006} & \benchval{0.507}{0.006} & \benchval{0.506}{0.006} \\
GL-10D & \benchval{0.996}{0.003} & \benchval{0.999}{0.001} & \benchval{0.598}{0.035} & \benchval{0.539}{0.025} & \benchval{0.570}{0.094} & \benchval{\textbf{0.536}}{0.023} & \benchval{0.531}{0.025} \\
GL-20D & \benchval{0.999}{0.001} & \benchval{1.000}{0.000} & \benchval{0.614}{0.031} & \benchval{\textbf{0.526}}{0.009} & \benchval{0.532}{0.039} & \benchval{\textbf{0.526}}{0.013} & \benchval{0.525}{0.012} \\
BCI & \benchval{0.946}{0.015} & \benchval{0.789}{0.061} & \benchval{0.553}{0.018} & \benchval{\textbf{0.540}}{0.017} & \benchval{0.542}{0.018} & \benchval{0.542}{0.018} & \benchval{0.542}{0.018} \\
\midrule
\multicolumn{8}{c}{\textbf{MMTV}} \\
\midrule
Two Moons & \benchval{0.510}{0.147} & \benchval{0.062}{0.009} & \benchval{0.159}{0.081} & \benchval{0.121}{0.052} & \benchval{\textbf{0.060}}{0.008} & \benchval{\textbf{0.060}}{0.008} & \benchval{0.060}{0.008} \\
SLCP & \benchval{0.498}{0.075} & \benchval{0.205}{0.269} & \benchval{0.139}{0.059} & \benchval{0.107}{0.042} & \benchval{\textbf{0.072}}{0.025} & \benchval{0.073}{0.026} & \benchval{0.071}{0.025} \\
OUP & \benchval{0.322}{0.085} & \benchval{0.046}{0.012} & \benchval{0.073}{0.022} & \benchval{0.048}{0.023} & \benchval{\textbf{0.045}}{0.013} & \benchval{\textbf{0.045}}{0.013} & \benchval{0.043}{0.010} \\
GL-10D & \benchval{0.636}{0.044} & \benchval{0.965}{0.103}\rlap{\textsuperscript{$\dagger$}} & \benchval{0.110}{0.033} & \benchval{0.062}{0.030} & \benchval{0.100}{0.136} & \benchval{\textbf{0.053}}{0.017} & \benchval{0.044}{0.008} \\
GL-20D & \benchval{0.640}{0.030} & \benchval{0.900}{0.125}\rlap{\textsuperscript{$\dagger$}} & \benchval{0.127}{0.064} & \benchval{0.043}{0.013} & \benchval{0.049}{0.061} & \benchval{\textbf{0.041}}{0.016} & \benchval{0.038}{0.002} \\
BCI & \benchval{0.485}{0.039} & \benchval{0.183}{0.206} & \benchval{0.074}{0.015} & \benchval{\textbf{0.061}}{0.012} & \benchval{0.062}{0.012} & \benchval{0.062}{0.012} & \benchval{0.062}{0.012} \\
\bottomrule
\end{tabular*}

\par\vspace{3pt}
\raggedright\fontsize{7}{8.5}\selectfont
\textdagger{} SIR MMTV excludes degenerate runs. It is defined for 80/90
model-level runs on GL-10D and 48/90 on GL-20D, yielding aggregate values for
all 30 GL-10D pairs and 25 of the 30 GL-20D pairs.
\par
\endgroup
\end{table}

\FloatBarrier
\subsection{Component allocation and weight estimation}
\label{app:mixture-allocation}

To isolate cases in which both posterior components remain relevant after
observing $x$, we stratify the 150 non-BCI pairs by their reference component
weights. Both components carry at least 5\% posterior mass in 18 pairs:
three Two Moons, four OUP, two SLCP, eight GL-10D, and one GL-20D pair.
In the remaining 132 pairs, one component carries more than 95\% of the
posterior mass. Only four pairs have at least 20\% mass in both components.
These subsets are descriptive diagnostics within the existing benchmark, not separate evaluation sets.

On this 18-pair subset, \pathspace{} has lower mean C2ST and MMTV
than \direct{} under both output configurations
(Table~\ref{tab:mixture-subset}). \direct{} has higher mean C2ST than PriorGuide
under both output configurations ($0.591$ and $0.595$ versus $0.568$ and
$0.549$). At $(100,0)$, this reversal is driven by four Gaussian Linear pairs
with large errors, and \direct{} is worse in only 4 of the 18 pairs.
\pathspace{} is more stable but still falls short of Spectra-Ref:
its mean C2ST is $0.529$ versus $0.519$, and its MMTV is higher on all 18
pairs, averaging $0.072$ versus $0.049$. This remaining gap shows that practical mixture accuracy is 
still sensitive to component-weight estimation.

\paragraph{BCI component allocation.}
The remaining 30 BCI pairs are treated separately because their reference
weights are finite-sample estimates from the SMC reference
(Appendix~\ref{app:references}). None has at least 5\% posterior mass in both
components, and the largest minority weight is $7\times10^{-4}$. \direct{} and
\pathspace{} select the reference-dominant component in all 90
pair--backbone combinations, and their practical C2ST and MMTV agree to three
decimal places. This extreme-allocation regime therefore says 
little about how accurately the very small minority weights are estimated.

\begin{table}[!htbp]
\centering
\small
\caption{
\textbf{Mixtures with substantial mass in both posterior components.}
Mean (\textcolor[gray]{0.48}{standard deviation}) across the 18 pairs for
which each component carries at least 5\% reference posterior mass.
Weight estimates are fixed across output configurations. Spectra-Ref uses reference
component weights and is shown only as a diagnostic at $(N,L)=(100,0)$.
}
\label{tab:mixture-subset}
\begin{tabular*}{\linewidth}{@{\extracolsep{\fill}}lrrrr@{}}
\toprule
 & \multicolumn{2}{c}{$(100,0)$} & \multicolumn{2}{c}{$(25,8)$} \\
\cmidrule(lr){2-3}\cmidrule(l){4-5}
Method & C2ST & MMTV & C2ST & MMTV \\
\midrule
PriorGuide & \benchval{0.568}{0.043} & \benchval{0.112}{0.038} & \benchval{0.549}{0.038} & \benchval{0.100}{0.044} \\
Spectra-Direct & \benchval{0.591}{0.123} & \benchval{0.158}{0.171} & \benchval{0.595}{0.122} & \benchval{0.161}{0.170} \\
Spectra-PS & \benchval{0.529}{0.022} & \benchval{0.072}{0.023} & \benchval{0.534}{0.026} & \benchval{0.074}{0.024} \\
Spectra-Ref & \benchval{0.519}{0.019} & \benchval{0.049}{0.012} & -- & -- \\
\bottomrule
\end{tabular*}

\end{table}

\begin{table}[!htbp]
\centering
\small
\caption{
\textbf{Accuracy as the number of prior components increases on Two Moons.}
Mean (\textcolor[gray]{0.48}{standard deviation}) across three observations
at $(N,L)=(25,8)$ after averaging three sampling seeds within each
backbone. \pathspace{} uses one backbone, while Spectra-Ref and PriorGuide use three backbones. All methods are evaluated against the support-matched
reference.
}
\label{tab:k-composition-accuracy}
\begin{tabular*}{\linewidth}{@{\extracolsep{\fill}}lrrrrrr@{}}
\toprule
 & \multicolumn{3}{c}{C2ST} & \multicolumn{3}{c}{MMTV} \\
\cmidrule(lr){2-4}\cmidrule(l){5-7}
$K$ & \pathspace{} & Spectra-Ref & PriorGuide & \pathspace{} & Spectra-Ref & PriorGuide \\
\midrule
2 & \benchval{0.517}{0.008} & \benchval{0.518}{0.009} & \benchval{0.573}{0.062} & \benchval{0.064}{0.016} & \benchval{0.070}{0.011} & \benchval{0.211}{0.152} \\
4 & \benchval{0.521}{0.013} & \benchval{0.517}{0.008} & \benchval{0.586}{0.049} & \benchval{0.074}{0.024} & \benchval{0.075}{0.011} & \benchval{0.260}{0.134} \\
8 & \benchval{0.520}{0.010} & \benchval{0.519}{0.009} & \benchval{0.555}{0.049} & \benchval{0.062}{0.004} & \benchval{0.069}{0.004} & \benchval{0.178}{0.141} \\
16 & \benchval{0.507}{0.001} & \benchval{0.509}{0.003} & \benchval{0.545}{0.038} & \benchval{0.063}{0.011} & \benchval{0.070}{0.002} & \benchval{0.170}{0.115} \\
\bottomrule
\end{tabular*}

\end{table}

\FloatBarrier
\subsection{Composition beyond two components}
\label{app:k-composition}

We next evaluate mixtures with more than two components. On Two Moons,
the target priors are equal-weight mixtures of
$K\in\{2,4,8,16\}$ Gaussian components with the strong-shift width. Their
centers are evenly spaced on a circle around the mild-prior center, with
radius one quarter of the training-box side, and each smaller component set
is contained in the larger ones. We evaluate accuracy up to $K=16$, beyond
which adjacent components are separated by less than one component standard
deviation and overlap heavily.

Sampling uses $(N,L)=(25,8)$ with three sampling seeds on three observations.
The support-matched quadrature reference also provides the reference
component weights. \pathspace{} uses the weight-estimation budget of the
two-component benchmark (Appendix~\ref{app:sampling-settings}) and one
backbone. Spectra-Ref and PriorGuide average over three backbones.

Both Spectra variants remain accurate as $K$ increases. 
Conditioning reduces the effective number of components relative to the 
nominal \(K\), but multiple components retain substantial 
posterior mass for the larger mixtures (Table~\ref{tab:k-composition-weights}). At $K=16$, the effective number of
components ranges from $2.5$ to $4.5$, and five components carry more than
5\% posterior mass for every observation. Path-space estimates closely track
the reference weights, with a maximum absolute error of
$1.6\times10^{-2}$ across all $K$, components, and observations. 
With only three observations, we treat this as a composition study 
rather than a method comparison.

\begin{table}[!htbp]
\centering
\small
\caption{
\textbf{Component allocation and path-space weight accuracy.}
For each of the three observations, the table reports the effective number
of components $1/\sum_k\alpha_k^2$ and the number carrying at least 5\% of
the posterior mass. The final column gives the largest absolute difference
between path-space and reference component weights across all components and
observations.
}
\label{tab:k-composition-weights}
\begin{tabular*}{\linewidth}{@{\extracolsep{\fill}}lcccc@{}}
\toprule
 & \multicolumn{2}{c}{Reference weights} & \multicolumn{2}{c}{Path-space estimate} \\
\cmidrule(lr){2-3}\cmidrule(l){4-5}
$K$ & Effective components & Components with at least 5\% & Effective components & Largest weight error \\
\midrule
2 & 1.00 / 1.02 / 1.00 & 1 / 1 / 1 & 1.00 / 1.02 / 1.00 & $3.4\times10^{-4}$ \\
4 & 1.19 / 1.36 / 1.12 & 2 / 2 / 2 & 1.19 / 1.35 / 1.12 & $2.2\times10^{-3}$ \\
8 & 1.21 / 3.09 / 2.18 & 2 / 3 / 2 & 1.21 / 3.07 / 2.18 & $1.6\times10^{-2}$ \\
16 & 2.50 / 4.54 / 4.10 & 5 / 5 / 5 & 2.50 / 4.52 / 4.15 & $8.4\times10^{-3}$ \\
\bottomrule
\end{tabular*}

\end{table}
\FloatBarrier

\section{Computational cost}
\label{app:compute-details}

This appendix reports single-factor online sampling time, the first-use cost
of mixture adaptation, and scaling with the number of prior components.

The timing experiments use the same Spectra implementation as the accuracy
experiments. For PriorGuide, we report both the benchmarked implementation
and the vector--Jacobian-product (VJP) implementation described in
Appendix~\ref{app:priorguide-variants}. The latter evaluates the same guidance
update without explicitly forming the full denoiser Jacobian and serves as
the strengthened timing baseline. Across the 22 timed configurations,
PriorGuide (VJP) is $2.2$--$3.2\times$ slower than Spectra, compared with
$3.5$--$19.7\times$ for the benchmarked implementation. These timings are
specific to the implementations and hardware used here.

\paragraph{Timing protocol.}
Unless stated otherwise, online sampling times are medians of five
synchronized repeats for 1,000 posterior samples on a single A100.
Compilation and input preparation are excluded from all timed repeats.
Method-specific setup costs are reported separately when first-use cost is
considered in Appendix~\ref{app:mixture-cost}. Each timing uses one
observation and one backbone per task.

\subsection{Single-factor accuracy and online cost}
\label{app:k1-cost}

Figure~\ref{fig:timing-five-tasks} reports online sampling time on Two Moons,
SLCP, BCI, GL-10D, and GL-20D at the sampler configurations used in the main benchmark. Both
the benchmarked and VJP implementations of PriorGuide are shown.

\begin{figure}[!htb]
\centering
\includegraphics[width=\linewidth]{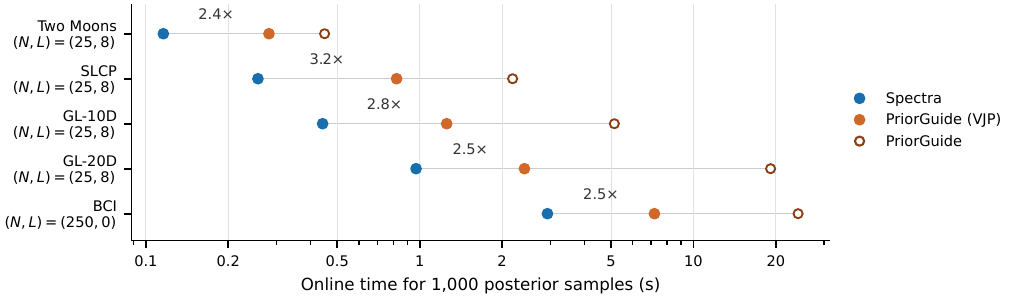}
\caption{
\textbf{Online sampling time across the five timed tasks.}
Each task uses the sampler configuration of the main benchmark and the timing
protocol above. Labels show the runtime ratio of PriorGuide (VJP) to Spectra.
}
\label{fig:timing-five-tasks}
\end{figure}

Figure~\ref{fig:accuracy-compute-fullcov} shows the MMTV--time trade-off and
adds \pgfullcov{} to the methods in Figure~\ref{fig:accuracy-compute}. The strong-shift development study uses
three observations, one backbone, and three sampling seeds, and evaluates
$(N,L)\in\{(25,0),(25,8),(50,2),(100,0),(100,2),(250,0)\}$.
Accuracy is averaged over seeds and then over the three observations, with
observation-level standard errors. These are the development observations
used for configuration selection. Table~\ref{tab:k1-heldout} reports the
held-out evaluation. Online time is measured on one of the three observations
using the same backbone.

\begin{figure}[!htb]
\centering
\includegraphics[width=\linewidth]{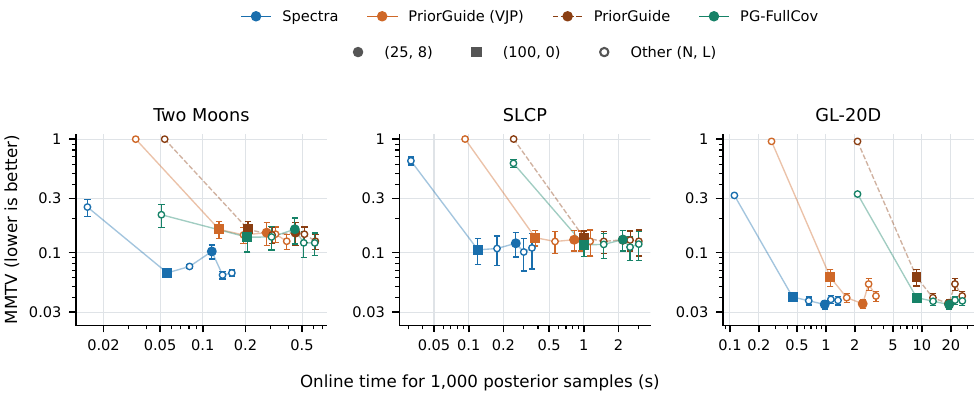}
\caption{
\textbf{Single-factor MMTV versus online sampling time.}
The figure adds \pgfullcov{} to the methods shown in
Figure~\ref{fig:accuracy-compute}. Filled circles mark $(25,8)$, squares mark
$(100,0)$, and hollow circles mark the remaining sampler configurations.
Error bars are observation-level standard errors. Both axes are logarithmic.
}
\label{fig:accuracy-compute-fullcov}
\end{figure}

\FloatBarrier
\subsection{Multi-component first-use cost}
\label{app:mixture-cost}

Figure~\ref{fig:mixture-cost-fullcov} compares the cost of the first 1,000
mixture posterior samples for PriorGuide, \pgfullcov{}, \direct{}, and \pathspace{}.
Each task uses two prespecified prior--observation pairs, one
backbone, and $(N,L)=(100,0)$. Stage costs are summarized by the median
across the two pairs. Each stage uses five timed repeats, except path-space
weight estimation, which uses three. Compilation and input preparation are
excluded.

\begin{figure}[htbp]
\centering
\includegraphics[width=\linewidth]{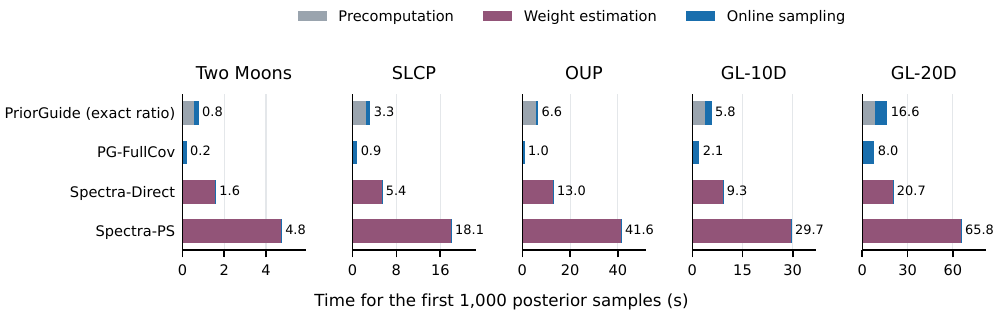}
\caption{
\textbf{Cost of the first 1,000 mixture posterior samples.}
Stacks separate precomputation, component-weight estimation, and online
sampling for PriorGuide, \pgfullcov{}, \direct{}, and \pathspace{}. Labels give
the total displayed cost. Panels use different linear time scales.
}
\label{fig:mixture-cost-fullcov}
\end{figure}

\direct{} and \pathspace{} share the same inexpensive component sampler,
but weight estimation can dominate the first adaptation. On SLCP, for
example, \pathspace{} takes $0.11$ seconds to draw 1,000 samples once
the weights are available, but $18.09$ seconds including weight estimation,
compared with $3.28$ seconds for PriorGuide including precomputation. Replacing PriorGuide's online stage with the VJP implementation lowers its
first-use cost further, but \pathspace{} remains the most expensive
first-use method on all five tasks.

Once estimated, the normalizers can be reused for the same observation and
component set when only the prior mixture coefficients change. These
measurements cover a single adaptation and therefore do not quantify the
amortized cost over a sequence of prior changes.

\FloatBarrier
\subsection{Scaling with the number of components}
\label{app:k-scaling}

Table~\ref{tab:k-scaling} measures online sampling time on Two Moons as the
number of prior components increases from $K=1$ to $K=64$. The target priors
are equal-weight Gaussian mixtures with the strong-shift component width.
Their centers lie on a common circle, and each smaller component set is
contained in the larger ones. For $K=1$, the single component is placed at
the center of the circle. Timing uses $(N,L)=(25,8)$ and the protocol above.
Component selection and PriorGuide precomputation are outside the timed
region.

\begin{table}[!htbp]
\centering
\small
\caption{
\textbf{Online sampling time as the number of prior components increases on
Two Moons.}
Median seconds for 1,000 posterior samples at $(N,L)=(25,8)$ under the timing
protocol of Appendix~\ref{app:compute-details}.
}
\label{tab:k-scaling}
\begin{tabular*}{\linewidth}{@{\extracolsep{\fill}}lrrrrrrr@{}}
\toprule
$K$ & 1 & 2 & 4 & 8 & 16 & 32 & 64 \\
\midrule
Spectra & 0.1166 & 0.1165 & 0.1165 & 0.1165 & 0.1165 & 0.1165 & 0.1164 \\
PriorGuide & 0.4464 & 0.4440 & 0.4442 & 0.4446 & 0.4456 & 0.4466 & 0.4482 \\
\bottomrule
\end{tabular*}

\end{table}

Spectra's online sampling time varies by less than $0.2\%$ from $K=1$ to
$K=64$, because each trajectory follows only the transported score of its
selected component. PriorGuide combines all $K$ components at every guidance
step, but its time changes by less than $1\%$ on this two-dimensional task,
where network evaluation and the score Jacobian dominate the per-step cost.
The PriorGuide-to-Spectra time ratio therefore remains between $3.81$ and
$3.85$.

Table~\ref{tab:k-scaling} uses the benchmarked PriorGuide implementation,
rather than the VJP implementation used for the single-factor speed
comparisons in Section~\ref{sec:compute-exp}.
The separate path-space weight-estimation stage runs proposal trajectories
for every component, so its cost grows linearly with $K$.

\section{Adaptation beyond the exact single-factor class}
\label{app:component-approx-results}

The main benchmarks use prior changes whose factors are exactly transportable.
This appendix applies the positive-mixture approximation of
Appendix~\ref{app:component-approx} to two Gaussian targets outside that class,
an anisotropy sweep and a rotated correlated prior, both on Two Moons.

\subsection{Controlled anisotropy boundary}
\label{app:anisotropy-boundary}

We start from the narrowest isotropic target in
Figure~\ref{fig:severity-sweep}, whose marginal standard deviation
$s=0.2/\sqrt3$ is one fifth of the training-prior value, and make it
anisotropic while keeping its center and determinant fixed,
\[
\Sigma_{\gamma,+}=\operatorname{diag}\!\left(\frac{s^2}{\gamma},\,s^2\gamma\right),
\qquad
\Sigma_{\gamma,-}=\operatorname{diag}\!\left(s^2\gamma,\,\frac{s^2}{\gamma}\right),
\qquad
\gamma\in\{1,2,4\}.
\]
The condition number is $\gamma^2$. At $\gamma=1$ the target is a single exact
factor. For $\gamma>1$ we use the construction of
Appendix~\ref{app:component-approx} with $\lambda=0.9\,\lambda_{\min}(\Sigma)$
and $h=6$, which gives $K=52$ factors at $\gamma=2$ and $K=104$ at $\gamma=4$.
The ten observations, three backbones, and samplers are those of
Figure~\ref{fig:severity-sweep}, with $(N,L)=(25,8)$ for PriorGuide and Spectra.

The approximation error grows with $\gamma$. On the training support, the prior
TV distance to the exact target is $0.00149$ at $\gamma=2$ and $0.03312$ at
$\gamma=4$. The mean posterior TV distance is $0.00306$ and $0.00356$ for the
two $\gamma=2$ orientations and $0.02360$ and $0.02889$ for $\gamma=4$, and the
reference weights spread over $3.46$ to $5.49$ effective components. This error
is small relative to the differences between methods. Scoring the Spectra-Ref samples against
the approximating target instead of the exact one changes mean C2ST by at most
$0.0015$ and mean MMTV by at most $0.0070$.

Component-weight estimation is the less reliable step. On nine of the ten observations,
Spectra-PS and Spectra-Ref differ in C2ST by at most $0.003$ in every setting.
On the remaining observation, path-space estimation places nearly all posterior
mass on components with negligible reference weight in the $+$ orientation, and
C2ST worsens by $0.417$ at $\gamma=2$ and $0.405$ at $\gamma=4$. In the $-$
orientation, the same observation degrades less, by $0.059$ and $0.018$. The
error in
the estimated normalizers may come from the learned finite base chain, from
Monte Carlo error in the estimator, or from both, and our diagnostics do not
separate them.

PriorGuide, although given the exact anisotropic ratio, is worse than
Spectra-Ref in both metrics on every observation in all four anisotropic
settings, and worse than Spectra-PS on at least nine of the ten. SIR is worse
than Spectra-Ref in C2ST on every observation but has the lower MMTV on five to
eight of the ten, as on Two Moons in Table~\ref{tab:k1-results}. Because the
target posterior changes with $\gamma$, values should not be compared across
rows.

\begin{table}[!htb]
\centering
\small
\caption{\textbf{Anisotropy sweep on Two Moons.}
Mean (standard deviation) over ten observations, scored against support-matched
references for the exact target. Each observation's value averages three
backbones, and for PriorGuide and Spectra three sampling seeds per backbone.
Sampling uses $(N,L)=(25,8)$, except $(100,0)$ for SIR. At $\gamma=1$, $K=1$ and
Spectra-Ref and Spectra-PS coincide.}
\label{tab:anisotropy-boundary}
\begin{tabular*}{\linewidth}{@{\extracolsep{\fill}}ccrrrrr@{}}
\toprule
$\gamma$ & Orient. & $K$ & SIR & PriorGuide & Spectra-Ref & Spectra-PS \\
\midrule
\multicolumn{7}{@{}l}{\textit{C2ST}} \\
1 & iso & 1 & \benchval{0.613}{0.069} & \benchval{0.621}{0.055} & \benchval{0.548}{0.016} & \benchval{0.548}{0.016} \\
2 & $+$ & 52 & \benchval{0.589}{0.090} & \benchval{0.621}{0.092} & \benchval{0.528}{0.013} & \benchval{0.570}{0.143}\rlap{\textsuperscript{$\dagger$}} \\
2 & $-$ & 52 & \benchval{0.578}{0.044} & \benchval{0.604}{0.063} & \benchval{0.528}{0.012} & \benchval{0.535}{0.027} \\
4 & $+$ & 104 & \benchval{0.599}{0.070} & \benchval{0.609}{0.084} & \benchval{0.531}{0.023} & \benchval{0.571}{0.149}\rlap{\textsuperscript{$\dagger$}} \\
4 & $-$ & 104 & \benchval{0.590}{0.073} & \benchval{0.600}{0.084} & \benchval{0.527}{0.018} & \benchval{0.530}{0.021} \\
\midrule
\multicolumn{7}{@{}l}{\textit{MMTV}} \\
1 & iso & 1 & \benchval{0.121}{0.121} & \benchval{0.310}{0.118} & \benchval{0.127}{0.021} & \benchval{0.127}{0.021} \\
2 & $+$ & 52 & \benchval{0.123}{0.171} & \benchval{0.311}{0.191} & \benchval{0.083}{0.016} & \benchval{0.166}{0.272}\rlap{\textsuperscript{$\dagger$}} \\
2 & $-$ & 52 & \benchval{0.116}{0.095} & \benchval{0.273}{0.140} & \benchval{0.085}{0.014} & \benchval{0.099}{0.050} \\
4 & $+$ & 104 & \benchval{0.088}{0.076} & \benchval{0.272}{0.180} & \benchval{0.072}{0.019} & \benchval{0.159}{0.287}\rlap{\textsuperscript{$\dagger$}} \\
4 & $-$ & 104 & \benchval{0.096}{0.065} & \benchval{0.251}{0.171} & \benchval{0.073}{0.011} & \benchval{0.079}{0.017} \\
\bottomrule
\end{tabular*}

\par\vspace{4pt}
\raggedright\footnotesize
\textsuperscript{$\dagger$}\ Driven by one observation, on which Spectra-PS is
worse than Spectra-Ref in C2ST by $0.417$ ($\gamma=2$) and $0.405$ ($\gamma=4$).
On the other nine, the two differ by at most $0.003$.
\end{table}

\subsection{Rotated correlated Gaussian example}
\label{app:rotated-correlated-example}

We apply the quadrature construction of
Appendix~\ref{app:component-approx} to a correlated Two Moons target prior
$\mathcal N(\mu,\Sigma)$ with $\mu=(0.3,-0.3)$ and $\Sigma$ obtained by
rotating $\operatorname{diag}(0.04,\,0.04/6)$ by $45^\circ$.
We use $\lambda=0.9\,\lambda_{\min}(\Sigma)=0.006$ and $h=6$, which give
Gauss--Hermite orders 16 and 4 and hence $K=64$ factors.
Posterior sampling uses $(N,L)=(100,0)$ with three seeds of 1,000 samples.

\begin{table}[!htb]
\centering
\small
\caption{
\textbf{Correlated-prior approximation on Two Moons.}
Approx.\ TV is the posterior total variation distance induced by the
$K=64$ finite quadrature, evaluated on a $4096\times4096$ grid.
C2ST is mean $\pm$ standard deviation over three sampling seeds. Approx. TV quantifies posterior error from the finite prior approximation, 
while C2ST evaluates the sampled posterior against the exact target.
PriorGuide and \pgfullcov{} use the exact ratio; Spectra-Ref and Spectra-PS
use the quadrature approximation with reference and path-space component
weights, respectively.
}
\label{tab:correlated-prior}
\begin{tabular}{lccccc}
\toprule
Observation & Approx.\ TV & PriorGuide & \pgfullcov{} & Spectra-Ref & Spectra-PS \\
\midrule
1 & 0.00306 & $0.535\pm0.003$ & $0.527\pm0.007$ & $0.507\pm0.006$ & $0.513\pm0.006$ \\
2 & 0.00456 & $0.519\pm0.011$ & $0.519\pm0.013$ & $0.518\pm0.008$ & $0.518\pm0.006$ \\
3 & 0.00285 & $0.549\pm0.008$ & $0.526\pm0.009$ & $0.530\pm0.009$ & $0.529\pm0.010$ \\
\bottomrule
\end{tabular}

\end{table}

The same quadrature is used for all three observations.
Spectra-Ref computes component weights from the quadrature posterior on the
grid, while Spectra-PS estimates all 64 weights using the path-space estimator
and the budget of Appendix~\ref{app:sampling-settings}.
Across three observations and three backbones, the estimated weights are
within $0.026$--$0.092$ in $L_1$ distance of the reference weights.
Replacing reference weights by path-space estimates changes C2ST by
$+0.0013\pm0.0020$ and MMTV by $+0.0035\pm0.0018$ after averaging over
backbones. This is a low-dimensional demonstration rather than a
scalable construction for correlated priors. Accurate path-space weights here
do not establish reliable weight estimation at large $K$, and in
Appendix~\ref{app:anisotropy-boundary} one observation fails at $K=52$ and
$K=104$.

\end{document}